\documentclass{article}

\usepackage[margin=1.0in]{geometry}
\usepackage[utf8]{inputenc} % allow utf-8 input
\usepackage[T1]{fontenc}    % use 8-bit T1 fonts
\usepackage{url}            % simple URL typesetting
\usepackage{booktabs}       % professional-quality tables
\usepackage{amsfonts}       % blackboard math symbols
\usepackage{nicefrac}       % compact symbols for 1/2, etc.
\usepackage{microtype}      % microtypography
\usepackage{xcolor}
\usepackage{amsmath}
\usepackage{amsthm}
\usepackage{mathtools}
\theoremstyle{definition}
\newtheorem{definition}{Definition}[section]

\usepackage{xr-hyper}
\makeatletter
\newcommand*{\addFileDependency}[1]{% argument=file name and extension
  \typeout{(#1)}
  \@addtofilelist{#1}
  \IfFileExists{#1}{}{\typeout{No file #1.}}
}
\makeatother

\usepackage{graphicx}
\usepackage[style=nature]{biblatex}
\usepackage{algorithm}
\usepackage[noend]{algpseudocode}
\usepackage{amsmath}
\usepackage{amsthm}
\usepackage{enumitem}
\usepackage{authblk}
\algnewcommand{\IIf}[1]{\State\algorithmicif\ #1\ \algorithmicthen}
\algnewcommand{\EndIIf}{\unskip\ \algorithmicend\ \algorithmicif}

\DeclareMathOperator*{\argmin}{arg\,min}
\DeclareMathOperator*{\argmax}{arg\,max}

\usepackage{todonotes}
\usepackage{datenumber}
\usepackage{eqname}
\usepackage{float}
\usepackage{setspace}

\newtheorem{theorem}{Theorem}

\title{Explaining a Series of Models by Propagating Shapley Values}

\author[1]{Hugh Chen}
\author[2]{Scott M. Lundberg}
\author[1,*]{Su-In Lee}

\affil[1]{{\small Paul G. Allen School of Computer Science and Engineering, University of Washington}}
\affil[2]{{\small Microsoft Research}}
\affil[*]{{\small Corresponding: suinlee@cs.washington.edu}}

\begin{document}

\setcounter{page}{1}

\date{}

{\setstretch{1}
\maketitle
}

\begin{abstract}
Local feature attribution methods are increasingly used to explain complex machine learning models.  However, current methods are limited because they are extremely expensive to compute or are not capable of explaining a distributed series of models where each model is owned by a separate institution.  The latter is particularly important because it often arises in finance where explanations are mandated.  Here, we present DeepSHAP, a tractable method to propagate local feature attributions through complex series of models based on a connection to the Shapley value.  We evaluate DeepSHAP across biological, health, and financial datasets to show that it provides equally salient explanations an order of magnitude faster than existing model-agnostic attribution techniques and demonstrate its use in an important distributed series of models setting.
\end{abstract}

\section{Introduction}
\label{sec:introduction}

With the widespread adoption of machine learning (ML), \textit{series of models} (i.e., where the outputs of predictive models are used as inputs to separate predictive models) are increasingly common.  Examples include: (1) \textit{stacked generalization}, a widely used technique \cite{wang2006using,healey2018mapping,bhatt2017improved,doumpos2007model,ottokaggle} to improve generalization performance by ensembling the predictions of many models (called base-learners) using another model (called a meta-learner) \cite{wolpert1992stacked}, (2) \textit{neural network feature extraction}, where models are trained on features extracted using neural networks \cite{guo1992classification,chen2016deep}, typically for structured data  \cite{xu2014deep,liang2017text,jahankhani2006eeg}, and (3) \textit{consumer scores}, where predictive models that describe a specific behavior (e.g., credit scores \cite{dixon2014scoring}) are used as inputs to downstream predictive models.  For example, a bank may use a model to predict customers’ loan eligibility on the basis of their bank statements and their credit score, which itself is often a predictive model \cite{fay_2020}.

Explaining a series of models is crucial for debugging and building trust, even more so because a series of models is inherently harder to explain compared to a single model.  One popular paradigm for explaining models are \textit{local feature attributions}, which explain why a model makes a prediction for a single sample (known as the ``explicand'' \cite{sundararajan2020themanyshapleyvalues}).  Existing \textit{model-agnostic} local feature attribution methods (e.g., IME \cite{strumbelj2010efficient}, LIME \cite{ribeiro2016should}, KernelSHAP \cite{lundberg2017unified}) work regardless of the specific model being explained. They can explain a series of models, but suffer from two distinct shortcomings: (1) their sampling-based estimates of feature importance are inherently variable, and (2) they have high computational cost which may not be tractable for large pipelines.  Alternatively, \textit{model-specific} local feature attribution methods (i.e. attribution methods that work for specific types of models) are often much faster than model-agnostic approaches, but generally cannot be used to explain a series of models.  Examples include those for (1) deep models (e.g., DeepLIFT \cite{deeplift}, Integrated Gradients \cite{sundararajan2017axiomatic}) and (2) tree models (e.g., Gain/Gini Importance \cite{breiman1984classification}, TreeSHAP \cite{treeshap}).

In this paper, we present DeepSHAP -- a local feature attribution method that is \textit{faster than model-agnostic methods} and \textit{can explain complex series of models that pre-existing model-specific methods cannot}.  DeepSHAP is based on connections to the Shapley value, a concept from game theory that satisfies many desirable axioms.  
% We focus on the axiom of \textit{efficiency}, which guarantees that local feature attributions sum up to a desirable value (e.g., the model output on the explicand minus the average model output across a set of baselines).  
We make several important contributions: 
\begin{enumerate}
    \item We propose a theoretical framework (Methods Section \ref{sec:methods:deeplift_shapley}) that connects the rules introduced in \citeauthor{deeplift} to the Shapley value with an interventional conditional expectation set function\footnote{with a flat causal graph} (ICE Shapley value) (Methods Section \ref{sec:methods:shapley}).
    % We explicitly describe the rules introduced in \cite{deeplift} as approximations to Shapley values based on an interventional conditional expectation (Figure \ref{fig:concept}a, Methods Section \ref{sec:methods:shapley}).  In particular, they can be seen as a \textit{k-partition} approximation to Shapley values for a non-linearity applied to a linear function that partition input features into $k$ disjoint sets, compute the Shapley values exactly and then propagate the attributions linearly within each set (Methods Section \ref{sec:methods:deeplift_shapley}).
    \item We show that the ICE Shapley value decomposes into an average over ``single baseline attributions''\footnote{\cite{merrick2020explanation} show an analogous result for an ``input distribution'' under an observational conditional expectation lift.} (Methods Section \ref{sec:methods:baseline}), where a single baseline attribution explains the model for a single sample (explicand) by comparing to a single sample (baseline).
    \item We propose a \textit{generalized rescale rule} to explain a complex series of models by propagating attributions while enforcing efficiency at each layer (Figure \ref{fig:concept}b, Methods Section \ref{sec:methods:series_of_models}).  This framework extends DeepSHAP to explain any series of models composed of linear, deep, and tree models.  
    % More generally, this framework will satisfy efficiency if the explanations for each model in the series satisfies efficiency.
    % In this paper our feature attributions are all based on approximating or exactly computing Shapley values based on an interventional conditional expectation \cite{janzing2019feature}.
    \item We propose a \textit{group rescale rule} to propagate local feature attributions to groups of features (Methods Section \ref{sec:methods:groups}).  We show that these group attributions better explain models with many features.  
    % Although the feature attribution for a group could be a summation over features within a group, if groups overlap or do not cover the full set of features, efficiency will not be satisfied.  
\end{enumerate}

% In order to incorporate a baseline distribution in a principled way, we view DeepSHAP attributions as an approximation to Shapley values (\hugh{Methods section}).  For Shapley values with an interventional conditional expectation we show that baseline distributions can be reduced to an average over single baseline games, where a single explicand (sample being explained) is compared to a single baseline (sample used as baseline) (\hugh{Methods section}).  First, we qualitatively demonstrate the downside of using a single baseline in an image data set as in DeepLIFT (Section \ref{sec:cifar_mult_ref}).  Then, we utilize an epidemiological data set to show that the baseline distribution is actually an important parameter that determines the scientific question implicit in the feature attributions (Section \ref{sec:nhanes_mult_ref}).

Many feature attribution methods must define the absence of a feature, often by masking features according to a single baseline sample (single baseline attribution) \cite{deeplift,sundararajan2017axiomatic,sundararajan2020themanyshapleyvalues}. In contrast, we show that under certain assumptions, the correct approach is to use many baseline samples instead (Appendix Section \ref{sec:supp:int_baseline_dist_proof}).  Qualitatively, we show that using many baselines avoids bias that can be introduced by single baseline attributions (Section \ref{sec:cifar_mult_ref}).  Additionally, we show that the choice of baseline samples is a useful parameter which changes the question answered by the attributions (Figure \ref{fig:concept}c, Section \ref{sec:nhanes_mult_ref}).

We qualitatively and quantitatively evaluate DeepSHAP in real-world datasets including biological, health, image, and financial data sets.  In the biological datasets \cite{a2012overview,bennett2018religious,curtis2012genomic,pereira2016somatic}, we qualitatively assess group feature attributions based on gene sets identified in prior literature (Section \ref{sec:pathway}).  In the health, image, and financial datasets \cite{cox1998plan,lecun1998mnist,heloc}, we quantitively show that DeepSHAP provides useful explanations and is drastically faster than model agnostic approaches using an ablation test, where we hide features according to their attribution values (Sections \ref{sec:loss}, \ref{sec:feature_extraction}, \ref{sec:stacked_generalization}).

% , we explain four important examples of series of models (Figure \ref{fig:concept}d).  (1) We estimate group (gene set) feature attributions for models trained to classify (i) a binary Alzheimer's phenotype based on 15,128 genes across 651 samples \cite{a2012overview,bennett2018religious} and (ii) a binary breast cancer phenotype (whether histologic grade exceeds or equals three) based on 24,368 genes across 2,199 samples \cite{curtis2012genomic,pereira2016somatic} (Section \ref{sec:pathway}).  
% (2) We show that DeepSHAP enables explanations of the model's loss rather than its output and conveys information useful for debugging a five year mortality prediction model trained with 153 epidemiological features across 35,854 samples (NHANES 1999-2014),  (Section \ref{sec:loss}).  Furthermore, we show that compared to model-agnostic approaches, DeepSHAP is orders of magnitude faster and is equally explanatory in terms of an ablation test (Section \ref{sec:loss}).  (3)  We highlight the improved speed and comparable utility of DeepSHAP attributions relative to model-agnostic approaches for explaining an MNIST classification GBT model with features extracted from a CNN (70,000 28 by 28 pixel digit images in MNIST \cite{lecun1998mnist}) (Section \ref{sec:feature_extraction}).  Finally,We qualitatively and quantitatively validate a simulated distributed model setting where consumer score base models are fed into a meta model that predicts home equity line of credit (HELOC) application risk (Section \ref{sec:stacked_generalization}). 

In addition, DeepSHAP is the only approach we are aware of that enables explanations of a distributed series of models (where each model belongs to a separate institution).  Model-agnostic approaches do not work because they need access to every model in the series, but institutions cannot share models because they are proprietary.  One extremely prevalent example of distributed models are \textit{consumer scores} which exist for nearly every American consumer \cite{dixon2014scoring} (Section \ref{sec:stacked_generalization}).  In this setting, transparency is a critical issue, because opaque scores can hide discrimination or unfair practices.
% DeepSHAP is the only method we are aware of which can explain consumer scores in a distributed manner.

% In the ``distributed'' (i.e., models exist in independent financial entities) consumer scoring example, DeepSHAP is the only applicable explanation technique of which we are aware.  Model-agnostic explanation methods requires a single institution to have access to all models in the series, which may be a significant obstacle to regulatory demands for transparency \cite{dixon2014scoring}.  Instead, DeepSHAP enables attributions to the original features used by the consumer scoring models while allowing institutions to keep their proprietary models private.  Enabling model explanations in this distributed setting is paramount given the prevalence of consumer scores and the inherent opacity of models distributed across independent entities.  As evidence of the extensive use of consumer scores, a 2014 Federal Trade Commision (FTC) study found that a single data broker had thousands of data segments on nearly every consumer in the United States \cite{schmitz2014secret}.  DeepSHAP provides a natural way to provide explanations in this widespread, real world scenario.

% Finally, in the NHANES 1999-2014 data set we explain five bagged MLP models whose predictions are fed into a meta-model, showcasing the utility of obtaining explanations at each layer of the series of models (Section \ref{sec:stacked_generalization}).

% \hugh{Describe dependence/summary plot - maybe in context of explaining concept figure?}

\begin{figure}[!ht]
\includegraphics[width=\textwidth]{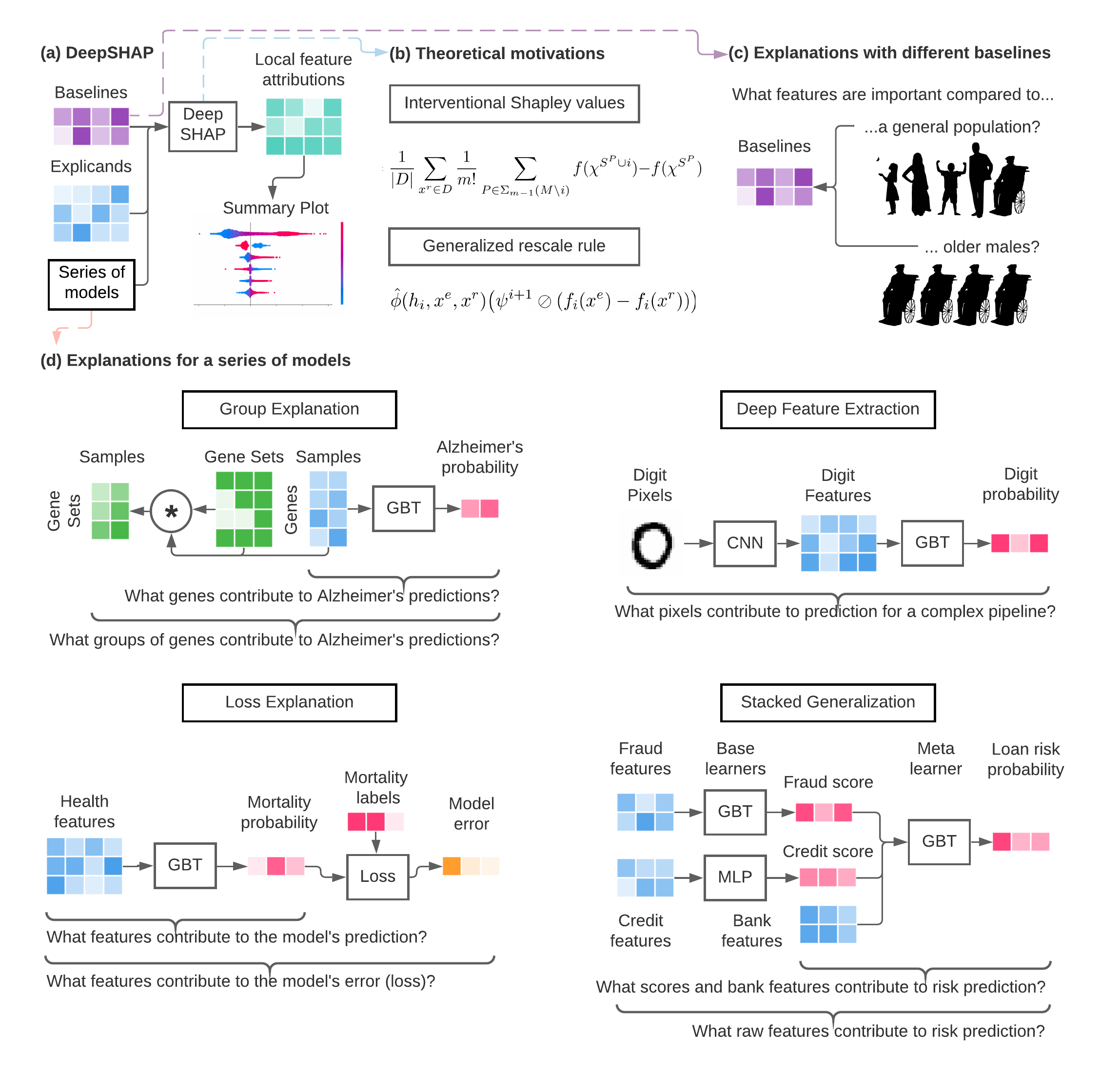}
\centering
\caption{\textbf{DeepSHAP estimates Shapley value feature attributions to explain a series of models using a baseline distribution.}  (a) Local feature attributions with DeepSHAP require explicands (samples being explained), a baseline distribution (samples being compared to), and a model that is comprised of a series of models.  They can be visualized to understand model behavior (Appendix Section \ref{sec:supp:feat_attr_plots}).  (b) Theoretical motivation behind DeepSHAP (Methods Sections \ref{sec:methods:shapley} and \ref{sec:methods:series_of_models}).  (c) The baseline distribution is an important, but often overlooked, parameter that changes the scientific question implicit in the local feature attributions we obtain.  (d) Explaining a series of models enables us to explain groups of features, model loss, and complex pipelines of models (deep feature extraction and stacked generalization).  Experimental setups are described in Appendix Section \ref{sec:supp:exp_setup}.}
\label{fig:concept}
\end{figure}

\section{Generalizing DeepSHAP local explanations}

% \hugh{Pick up here}

A closely related method to DeepSHAP was designed to explain deep models ($f:\mathbb{R}^m\to\mathbb{R}$) \cite{lundberg2017unified}, by performing DeepLIFT \cite{deeplift} using the average as a baseline.  However, using a single average baseline is not the correct approach to explain non-linear models based on connections to Shapley values with an interventional conditional expectation set function and a flat causal graph \cite{janzing2019feature}.  Instead, we show that the correct way to obtain the interventional Shapley value local feature attributions (denoted as $\phi(f,x^e)\in \mathbb{R}^m$) based on an explicand ($x^e\in \mathbb{R}^m$), or sample being explained, is to average over single baseline feature attributions (denoted as $\phi(f,x^e,x^b)\in \mathbb{R}^m$) where baselines are $x^b\in \mathbb{R}^m$ and $D$ is the set of all baselines (details in Methods Section \ref{sec:methods:baseline}):  
\begin{equation}
\phi(f,x^e)=\frac{1}{|D|}\sum_{x^b\in D} \phi(f,x^e,x^b)
\end{equation}

DeepLIFT \cite{deeplift} explains deep models by propagating feature attributions at each layer of the deep model.  Here, we extend DeepLIFT by generalizing DeepLIFT's rescale rule to accommodate more than neural network layers while guaranteeing layer-wise efficiency (details in Methods Section \ref{sec:methods:series_of_models}).  For a series of models which can be represented as a composition of functions ($f_k(x)=(h_k\circ \cdots \circ h_1)(x)$, where $h_i:\mathbb{R}^{m_i}\to\mathbb{R}^{o_i}$, $h_i=o_{i-1}\forall i\in2,\cdots k$, $h_1=m$, and $o_k=1$) with intermediary models ($f_i(x)=(h_i \circ \cdots \circ h_1)(x)$), DeepSHAP attributions are computed as:
\begin{align}
\psi^k&=\hat{\phi}(h_k,x^e,x^b)\\
\psi^i&=\hat{\phi}(h_i,x^e,x^b)\big(\psi^{i+1}\oslash(f_i(x^e)-f_i(x^b))\big),\text{ }i\in 1,\cdots, k-1.
\end{align}
We use Hadamard division to denote an element-wise division of $\Vec{a}$ by $\Vec{b}$ that accommodates zero division, where, if the denominator $b_i$ is 0, we set $a_i/b_i$ to 0.  The attributions $\hat{\phi}$ for a particular model in the stack are computed utilizing DeepLIFT with the rescale rule for deep models \cite{deeplift}, interventional TreeSHAP for tree models \cite{treeshap}, or exactly for linear models.  Each intermediate attribution $\psi^i$ serves as feature attribution that satisfies efficiency for $h_i$'s input features, where the attribution in the raw feature space is given by $\psi^1$.  This approach takes inspiration from the chain rule applied specifically to deep networks in \cite{deeplift}, that we extend to more general classes of models. 

\section{Incorporating a baseline distribution}

We now use DeepSHAP to explain deep models with different choices of baseline distributions to empirically evaluate our theoretical connections to interventional conditional expectations.

\subsection{Baseline distributions avoid bias}
\label{sec:cifar_mult_ref}

\begin{figure}[!ht]
\includegraphics[width=.8\textwidth]{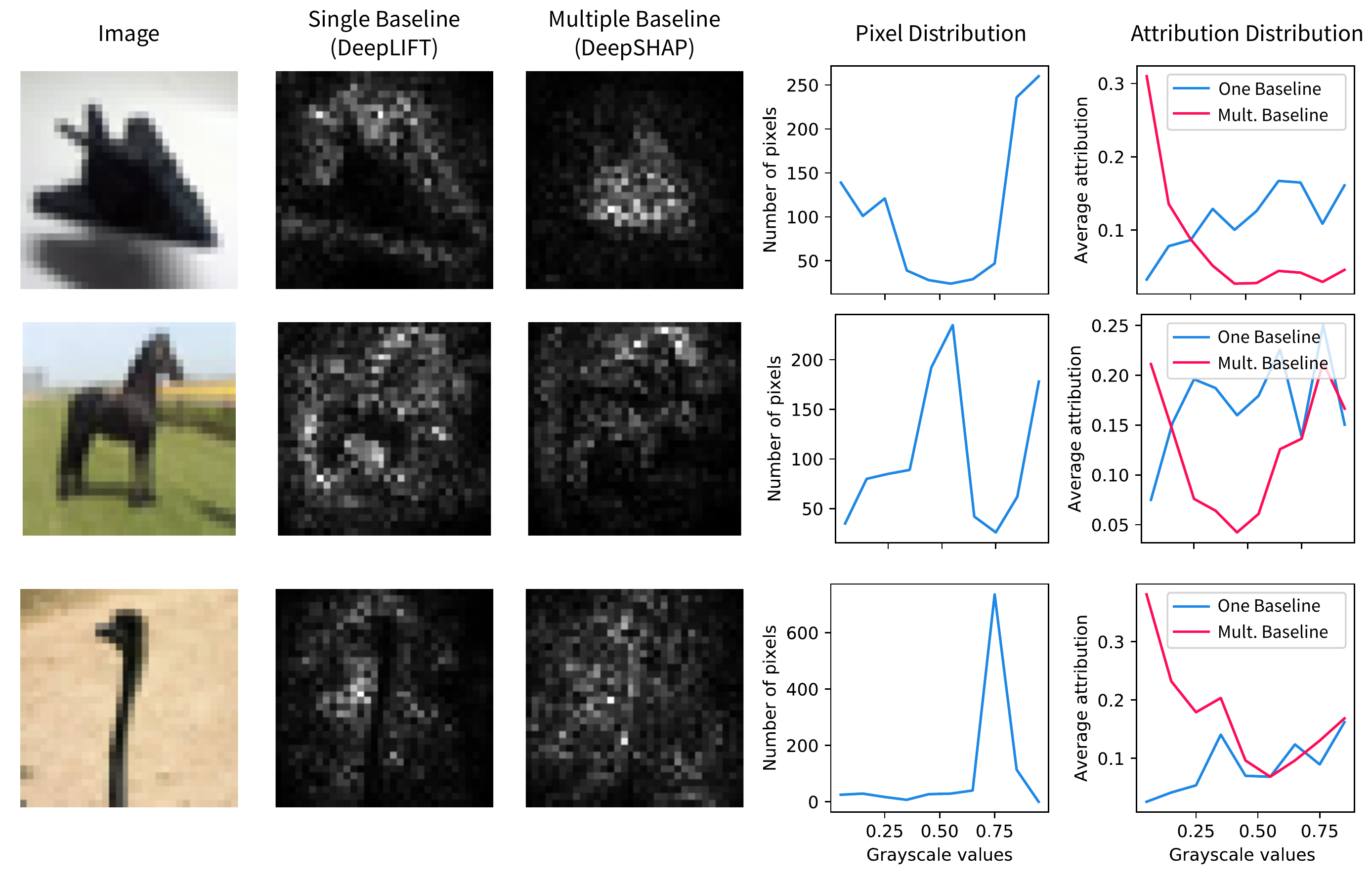}
\centering
\caption{\textbf{Using a single all-black baseline image (DeepLIFT) leads to biased attributions compared to attributions with a randomly sampled baseline distribution (DeepSHAP).}  The image is the explicand.  The attribution plots are the sum of the absolute value of the feature attributions for the three channels of the input image.  The pixel distribution is the distribution of pixels in terms of their grayscale values.  The attribution distribution is the amount of attribution mass upon a group of pixels binned by their grayscale values.}
\label{fig:mult_ref_cifar}
\end{figure}

We show that single baseline attributions are biased in a CNN that achieves 75.56\% test accuracy (hyperparameters in Appendix Section \ref{sec:supp:mult_ref_cifar_setup}) in the CIFAR10 data set \cite{krizhevsky2009learning}.  We aim to demonstrate that single baselines can lead to bias in explanations by comparing attributions using either a single baseline (an all-black image) as in DeepLIFT or a random set of 1000 baselines (random training images) as in DeepSHAP.  Although the black pixels in the image are qualitatively important, using a single baseline leads to biased attributions with little attribution mass for black pixels (Figure \ref{fig:mult_ref_cifar}).  In comparison, averaging over multiple baselines leads to qualitatively more sensible attributions.  Quantitatively, we show that despite the prevalence of darker pixels (pixel distribution plots in Figure \ref{fig:mult_ref_cifar}), single baseline attributions are biased to give them low attribution, whereas averaging over many baselines more sensibly assigns a large amount of credit to dark pixels (attribution distribution plots in Figure \ref{fig:mult_ref_cifar}).  To generalize this finding beyond DeepSHAP, we replicate this bias for IME and IG, two popular feature attribution methods that similarly rely on baseline distributions (Appendix Section \ref{sec:supp:mult_ref_cifar_ig_ime}).

\subsection{Natural scientific questions with baseline distributions}
\label{sec:nhanes_mult_ref}

\begin{figure}[!ht]
\includegraphics[width=.9\textwidth]{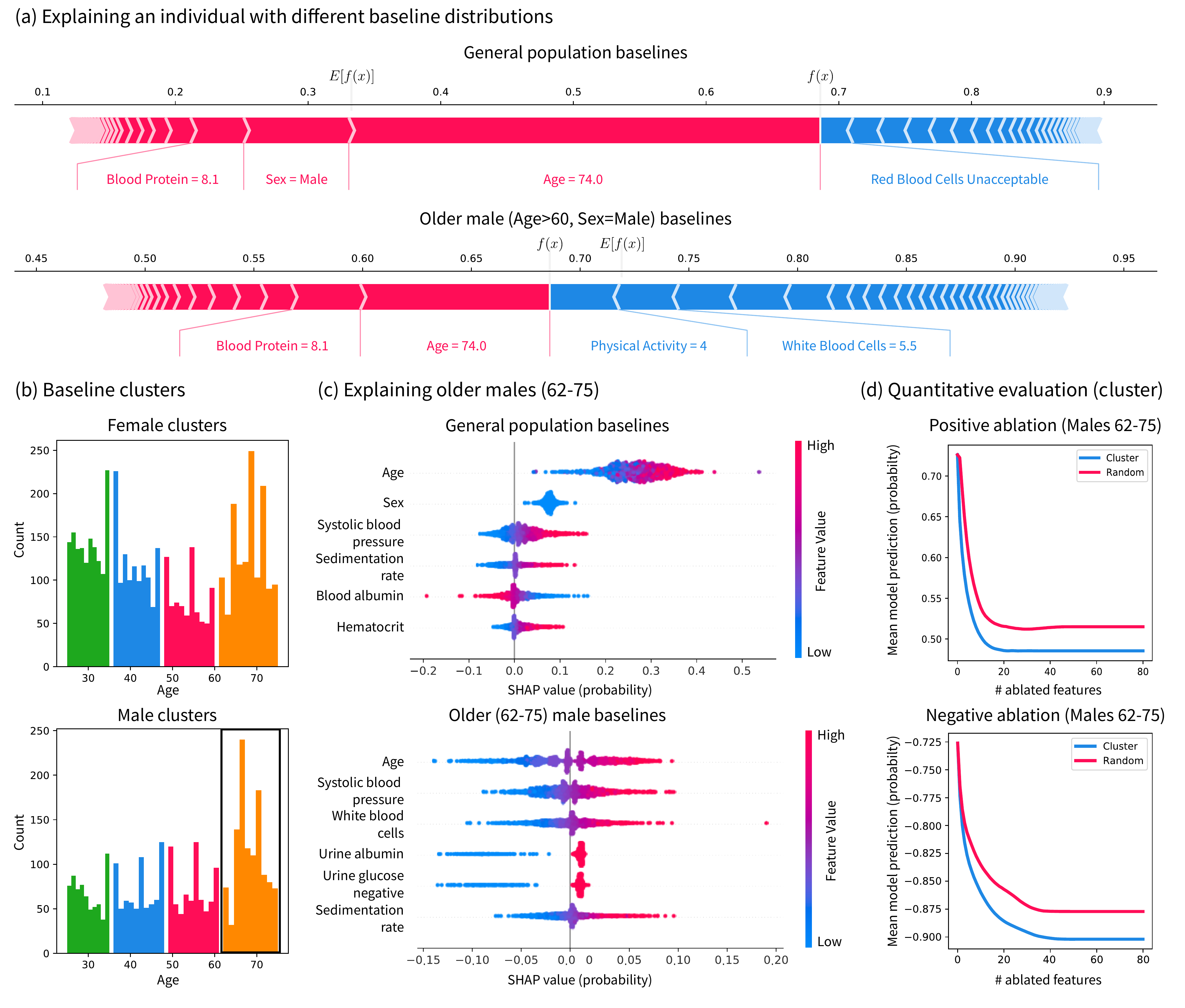}
\centering
\caption{\textbf{The baseline distribution is an important parameter for model explanation.} (a) Explaining an older male explicand with both a general population baseline distribution and an older male baseline distribution.  (b) Automatically finding baseline distributions using 8-means clustering on age and sex.  (c) Explaining the older male subpopulation (62-75 years old) with either a general population baseline or an older male baseline.  (d) Quantitative evaluation of the feature attributions via positive and negative ablation tests where we mask with the mean of the older male subpopulation. Note that (b) shows summary plots (Appendix Section \ref{sec:supp:summary_plot}) and (c) shows dependence plots (Appendix Section \ref{sec:supp:dependence_plot}).}
\label{fig:mult_ref_nhanes}
\end{figure}

To demonstrate the importance of baseline distributions as a parameter, we explain an MLP (hyperparameters in Appendix Section \ref{sec:supp:mult_ref_nhanes_setup}) with 0.872 ROC AUC for predicting fifteen year mortality in the NHANES I data set.
% (Appendix Section \ref{sec:supp:nhanes_i})  
We use DeepSHAP to explain an explicand relative to a baseline distribution drawn uniformly from all samples (Figure \ref{fig:mult_ref_nhanes}a (top)).  This explanation places substantial emphasis on age and sex because it compares the explicand to a population that includes many younger/female individuals.  However, in practice epidemiologists are unlikely to compare a 74-year old male to the general population.  Therefore, we can manually select a baseline distribution of older males to reveal novel insights, as in Figure \ref{fig:mult_ref_nhanes}a (bottom).  The impact of gender is gone because we compare only to males, and the impact of age is lower because we compare only to older individuals.  Furthermore, the impact of physical activity is much higher because being physically active is more healthy compared to older individuals, who are less active than the general population.  This example illustrates that the baseline distribution is an important parameter for feature attributions.

To provide a more principled approach to choosing the baseline distribution parameter, we propose k-means clustering to select a baseline distribution (detail in Methods Section \ref{sec:methods:baseline_kmeans}).  Previous work analyzed clustering in the attribution space or contrasting to negatively/positively labelled samples \cite{merrick2020explanation}.  In Figure \ref{fig:mult_ref_nhanes}b, we show clusters according to age and gender.  Then, we explain many older male explicands using either a general population or an older male population baseline distribution (Figure \ref{fig:mult_ref_nhanes}c).  When we compare to the older male baselines, the importance of age is centered around zero, sex is no longer important, and the importance orderings of remaining features change.  Further, the inquiry we make changes from ``What features are important for older males relative to a general population?'' to ``What features are important for older males relative to other older males?''.  To quantitatively evaluate whether our attributions answer the second inquiry, we can ablate features in order of their positive/negative importance by masking with the mean of the older male baseline distribution (Figure \ref{fig:mult_ref_nhanes}d, (Methods Section \ref{sec:methods:ablation_test})).  In both plots, lower curves indicate attributions that better estimated positive and negative importance.  \textit{For both tests, attributions with a baseline distribution chosen by k-means clustering substantially outperforms a baseline distribution drawn from the general population.}

We find that our clustering-based approach to selecting a baseline distribution has a number of advantages.  Our recommendation is to choose baseline distributions by clustering according to non-modifiable, yet meaningful, features like age and gender.  This yields explanations that answer questions relative to inherently interpretable subpopulations (e.g., older males).  The first advantage is that choosing baseline distributions in this way decreases variance in the features that determined the clusters and subsequently reduces their importance to the model.  This is desirable for age and gender because individuals typically cannot modify their age or gender in order to reduce their mortality risk.  Second, this approach could potentially reduce model evaluation on off-manifold samples when computing Shapley values \cite{kumar2020problems,frye2020shapley} by considering only baselines within a reasonable subpopulation.  The final advantage is that the flexibility of choosing a baseline distribution allows feature attributions to answer natural contrastive scientific questions \cite{merrick2020explanation} that improve model comprehensibility, as in Figure \ref{fig:mult_ref_nhanes}c. 

% \hugh{Describe the differences in explaining older males depending on different baseline distributions.  Once again de-emphasis of gender and centered age.  Decreased importance of sedimentation rate (biased distribution for older individuals).  Maybe hematocrit too?  Maybe look into urine albumin because the dependence plot is interesting.}

\section{Explaining a series of models}

DeepSHAP (and DeepLIFT) have been shown to be very fast and performant explanation methods for explaining deep models \cite{deeplift,schwab2019cxplain,sixt2019explanations}.  In this section, we instead focus on evaluating our extension of DeepSHAP to accommodate a series of mixed models (trees, neural networks, and linear models) and address four impactful applications.

\subsection{Group attributions identify meaningful gene sets}
\label{sec:pathway}

\begin{figure}[!ht]
\includegraphics[width=.9\textwidth]{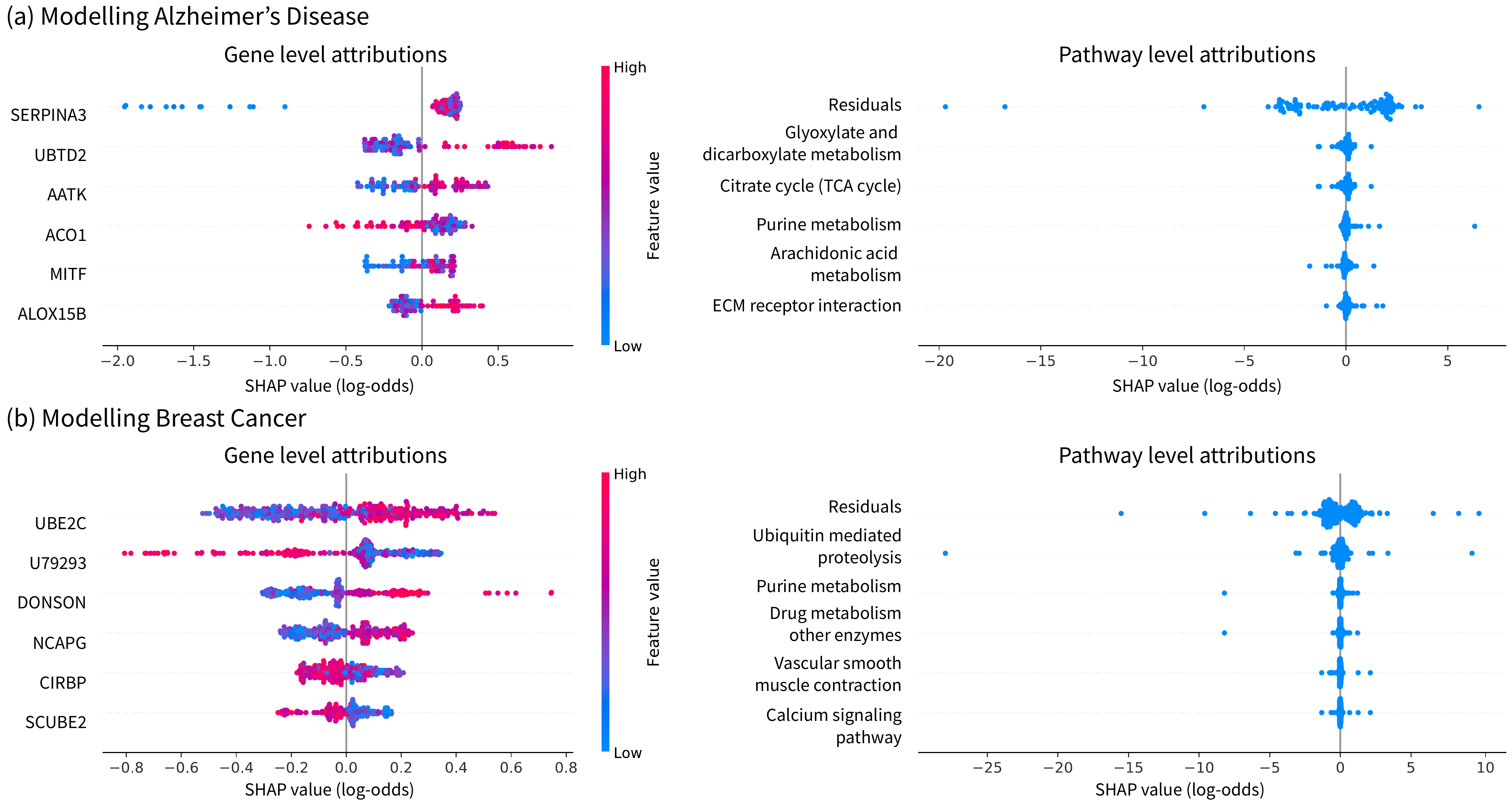}
\centering
\caption{\textbf{Propagating attributions to gene sets enables higher level understanding.}  (a) Gene and gene set attributions for predicting Alzheimer's disease using gene expression data.  (b) Gene and gene set attributions for predicting breast cancer tumor stage using gene expression data.  Residuals in the gene set attributions summarize contributions for genes that are not present in any gene set and describes variations in output not described by the pathways we analyzed.  Note that (a) and (b) show summary plots (Appendix Section \ref{sec:supp:summary_plot}).}
\label{fig:pathway_attributions}
\end{figure}

We explain two MLPs trained to predict Alzheimer's disease status and breast cancer tumor stage from gene expression data with test ROC AUC of 0.959 and 0.932, respectively.  We aim to demonstrate that our approach to propagating attributions to groups contributes to model interpretability by validating our discoveries with scientific literature.   Gene expression data is often extremely high dimensional; as such, solutions such as gene set enrichment analysis (GSEA) are widely used \cite{subramanian2005gene}.  In contrast, we aim to attribute importance to gene sets while maintaining efficiency by proposing a
\textit{group rescale rule} (Methods Section \ref{sec:methods:groups}).  This rule sums attributions for genes belonging to each group and then normalizes according to excess attribution mass due to multiple groups containing the same gene. It generalizes to arbitrary groups of features beyond gene sets, such as categories of epidemiological features (e.g., laboratory measurements, demographic measurements, etc.). 

We can validate several key genes identified by DeepSHAP.  For Alzheimer's disease, the overexpression of SERPINA3 has been closely tied to prion diseases \cite{vanni2017differential}, and UBTD2 has been connected to frontotemporal dementia -- a neurodenerative disorder \cite{taskesen2017susceptible}.  For breast cancer tumor stage, UBE2C was positively correlated with tumor size and histological grade \cite{mo2017clinicopathological}.  In addition to understanding gene importance, understanding higher level importance can be obtained using gene sets, i.e., groups of genes defined by biological pathways or co-expression.  We obtain gene set attributions by grouping genes according to a curated gene set.\footnote{Curated gene sets available here: https://www.gsea-msigdb.org/gsea/msigdb/collections.jsp\#C2}: the KEGG (Kyoto Encyclopedia of Genes and Genomes) pathway database (additional gene set attributions in Appendix Section \ref{sec:supp:additional_gene_sets})  

Next, we verify important gene sets identified by DeepSHAP.  For Alzheimer's disease, the glyoxylate and dicarboxylate metabolism pathway was independently identified based on metabolic biomarkers \cite{yu2017high}; several studies have demonstrated aberrations in the TCA cycle in Alzheimer's disease brain \cite{atamna2007mechanisms}; 
and alterations of purine-related metabolites are known to occur in early stages of Alzheimer's disease \cite{alonso2018purine}.  For breast cancer, many relevant proteins are involved in ubiquitin-proteasome pathways \cite{ohta2004ubiquitin} and purine metabolism was identified as a major metabolic pathway differentiating a highly metastatic breast cancer cell line from a slightly metastatic one \cite{kim2016comparative}.  \textit{Identifying these phenotypically relevant biological pathways demonstrates that our group rescale rule identifies important pathways.}

% One potential limitation of our approach is that larger size gene sets will be biased toward larger attribution mass by definition (Appendix Section \ref{sec:supp:gene_set_size_bias}).

\subsection{Loss attributions provide insights to model behavior}
\label{sec:loss}

\begin{figure}[!ht]
\includegraphics[width=.9\textwidth]{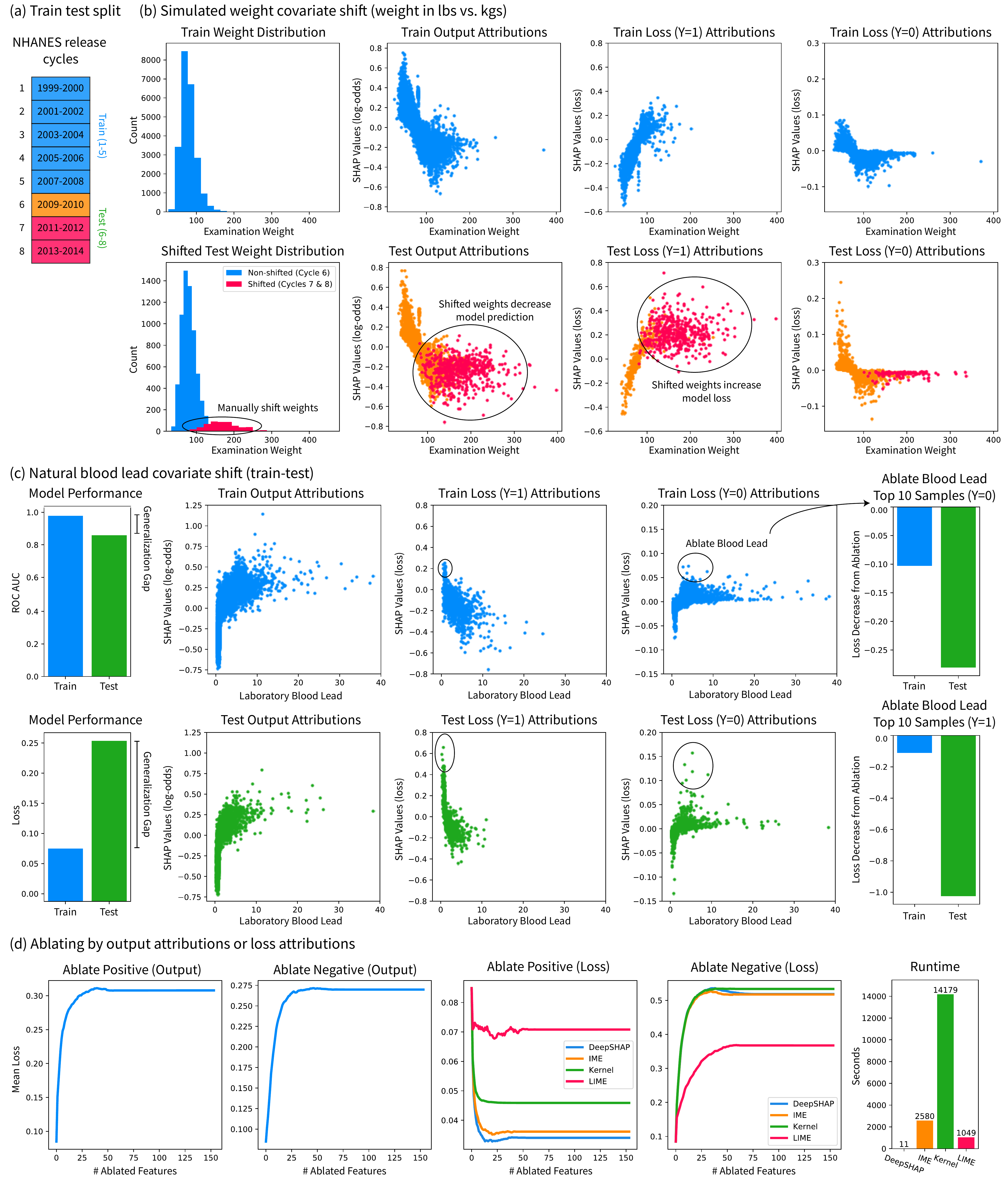}
\centering
\caption{\textbf{Explanations of the model's loss rather than the model's prediction yields new insights.} (a) We train on the first five cycles of NHANES (1999-2008) and test on the last three cycles (2009-2014).  (b) We identify a simulated covariate shift in cycles 7-8 (2011-2014) by examining loss attributions.  (c) Under a natural covariate shift, we identify and quantitatively validate test samples for which blood lead greatly increases the loss in comparison to training samples.  (d) We ablate output attributions (DeepSHAP) and loss attributions (DeepSHAP, IME, KernelSHAP, and LIME) to show their respective impacts on model loss.  We compare only to model-agnostic methods for loss attributions because explaining model loss requires explaining a series of models.  Note that (b) and (c) show dependence plots (Appendix Section \ref{sec:supp:dependence_plot}).  
% We do not compare to model-agnostic methods (IME, KernelSHAP, or LIME) because although they can explain the loss in theory, the existing packages only support explanations of the model outputs.  However, we would expect DeepSHAP to run much faster than model-agnostic approaches as in Figure \ref{fig:feature_extraction}.
}
\label{fig:loss_attributions}
\end{figure}

% \hugh{Section in appendix describing why we can't just modify the treexplainer algorithms to explain in probability or loss space (additivity)}

We examine an NHANES (1999-2014) mortality prediction GBT model (0.868 test set ROC AUC) to show how explaining the model's loss (loss explanations) provides important insights different from insights revealed by explaining the model's output (output explanations).  DeepSHAP lets us explain transformations of the model's output.  For instance, we can explain a binary classification model in terms of its log-odds predictions, its probability predictions (often easier for non-technical collaborators to understand; see Appendix Section \ref{sec:supp:nhanes_prob_vs_logodds}), or its loss computed based on the prediction.  Here, we focus on local feature attributions that explain per-sample loss.\footnote{This is analogous to model monitoring in \cite{treeshap}, which is in fact enabled via the generalized rescale rule we present here.}

We train our model on the first five release cycles of the NHANES data (1999-2008) and evaluate it on a test set of the last three release cycles (2009-2014) (Figure \ref{fig:mult_ref_nhanes}a).  As a motivating example, we simulate a covariate shift in the weight variable by re-coding it to be measured in pounds, rather than kilograms, in release cycles 7 and 8 (Figure \ref{fig:mult_ref_nhanes}b).  Then, we ask, ``Can we identify the impact of the covariate shift with feature attributions?''  Comparing the train and test output attributions, release cycles 7 and 8 are skewed, but they mimic the same general shape of the training set attributions.  If we did not color by release cycles, it might be difficult to identify the covariate shift.  In contrast, for loss attributions with positive labels, we can identify that the falsely increased weight leads to many misclassified samples where the loss weight attribution exceeds the expected loss.  Although such debugging is powerful, it is not perfect.  Note that in the negatively labelled samples, we cannot clearly identify the covariate shift because higher weights are protective and lead to more confident negative mortality prediction.

Next, we examine the natural generalization gap induced by covariate shift over time, which shows a dramatically different loss in the train and test sets (Figure \ref{fig:loss_attributions}c).  We can see that output attributions are similarly shaped between the train and test distributions; however, the loss attributions in the test set are much higher than in the training set.  We can quantitatively verify that negative blood lead affects model performance more in the test set by ablating blood lead for the top 10 samples in the train and test sets according to their loss distributions.  From this, we can see that blood lead constitutes a substantial covariate shift in the model's loss and helps explain the observed generalization gap.

As an extension of the quantitative evaluation in Figure \ref{fig:loss_attributions}c, we can visualize the impact on the model's loss of ablating by output attributions compared to ablating by loss attributions (Figure \ref{fig:loss_attributions}d).  This ablation test (Methods Section \ref{sec:methods:ablation_test}) asks ``What features are important to the model's performance (loss)?''  Ablating the positive and negative attributions both increase the mean model loss by hiding features central to making predictions.  However, ablating by the negative loss attribution directly increases the loss far more drastically than ablating by the output.  More so, ablating positive loss attributions clearly decreases the mean loss, which is not achievable by output attribution ablation.  Finally, we compare loss attributions computed using either a model-agnostic approach or DeepSHAP.  In this setting, DeepSHAP is two orders of magnitude faster than model-agnostic approaches (IME, KernelSHAP, and LIME) while showing extremely competitive positive loss ablation performance and the best negative loss ablation performance.

\subsection{Explaining deep image feature extractors}
\label{sec:feature_extraction}

\begin{figure}[!ht]
\includegraphics[width=.8\textwidth]{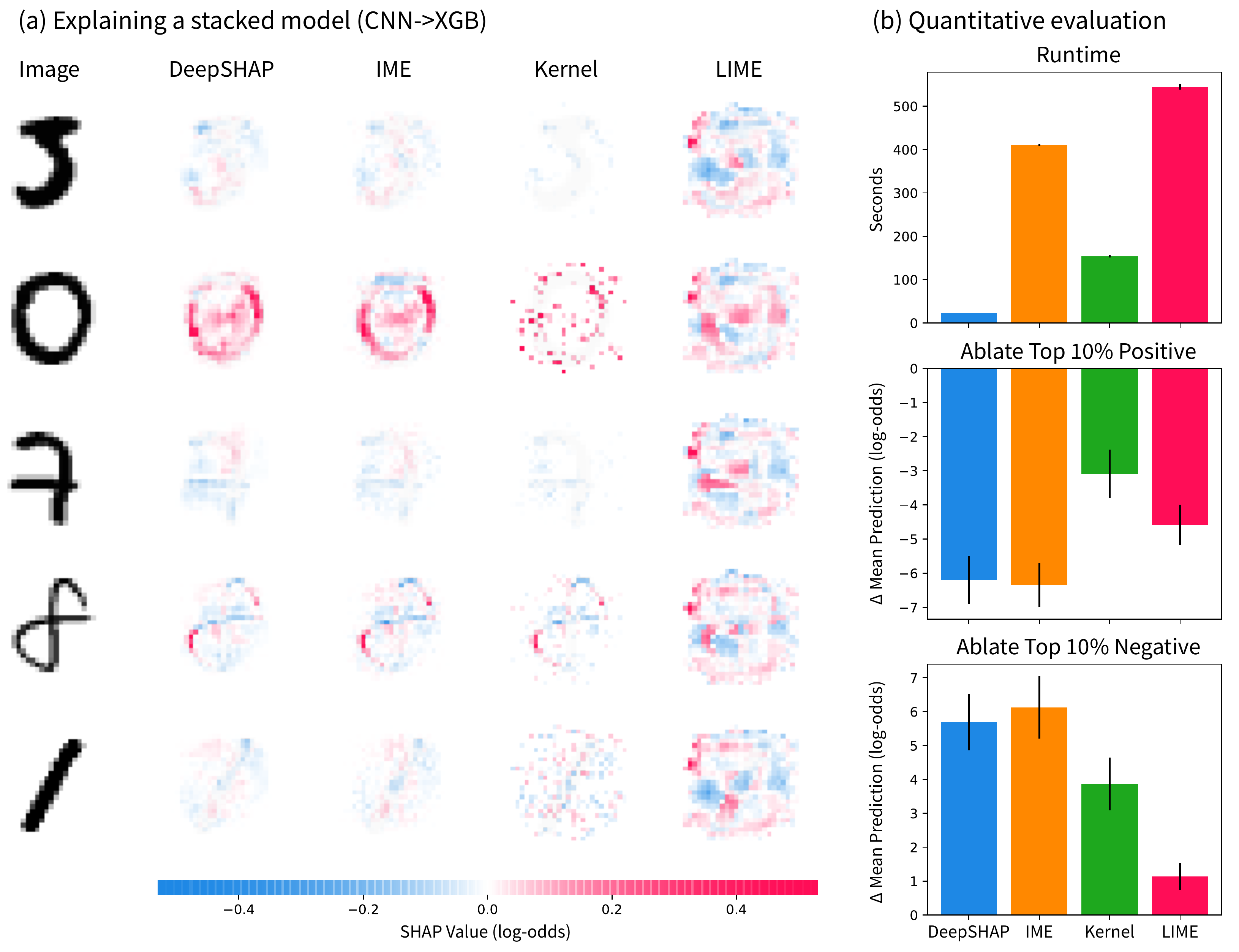}
\centering
\caption{\textbf{Explaining a series of models comprised of a convolutional neural network feature extractor and a gradient boosted tree classifier.} (a) Explanations from DeepSHAP and state of the art model-agnostic approaches.  Each model-agnostic approach has a ``number of samples'' parameter which we set to 100,000.  (b) Quantitative evaluation of approaches, including runtime and ablation of the top 10\% of positive and negative features.}
\label{fig:feature_extraction}
\end{figure}

We compare DeepSHAP explanations to a number of model-agnostic explanations for a series of two models: a CNN feature extractor fed into a GBT model that classifies MNIST zeros with 0.998 test accuracy.  In this example, non-linear transformations of the original feature space improve performance of the downstream model (Appendix Section \ref{sec:supp:feature_extraction_performance}) but make model-specific attributions impossible.  Qualitatively, we can see that DeepSHAP and IME are similar, whereas KernelSHAP is similar for certain explicands but not others\footnote{One potential reason is the regularization in the default settings of the package.} (Figure \ref{fig:feature_extraction}a).  Finally, LIME's attributions show the shape of the original digit, but there is a consistent attribution mass around the surrounding parts of the digit.  Qualitatively, we observe that the DeepSHAP attributions are sensible.  The pixels that constitute the zero digit and the absence of pixels in the center of the zero are important for a positive zero classification.\footnote{The importance of the absence of pixels in the center of the zero is revealed because we use a baseline distribution; and it would not be revealed with an all-white baseline image.}  

In terms of quantitative evaluations, we report the runtime and performance of the different approaches in Figure \ref{fig:feature_extraction}b.  We see that DeepSHAP is an order of magnitude faster than model-agnostic approaches, with KernelSHAP being the second fastest.  Then, we ablate the top 10\% of important positive or negative pixels to see how the model's prediction changes.  If we ablate positive pixels, we would expect the model's predictions to drop, and vice versa for negative pixels; doing both showed that DeepSHAP outperforms KernelSHAP and LIME, and performs comparably to IME at greatly reduced computational cost.

\subsection{Explaining distributed proprietary models}
\label{sec:stacked_generalization}

\begin{figure}[!ht]
\includegraphics[width=\textwidth]{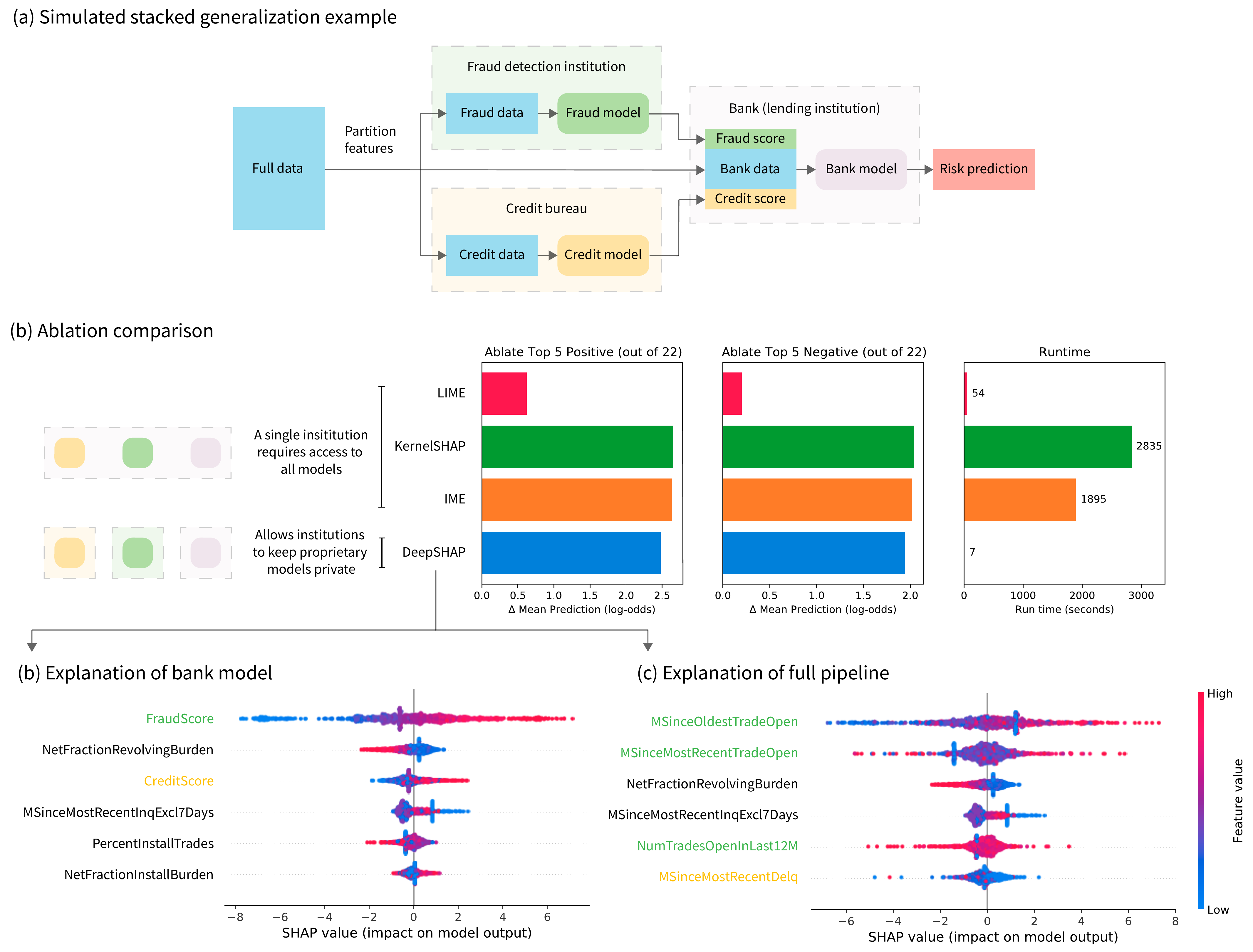}
\centering
\caption{\textbf{Explaining a stacked generalization pipeline of models for the HELOC data set (details in Appendix Section \ref{sec:supp:heloc})}  (a) A simulated model pipeline in the financial services industry.  We partition the original set of features into fraud, credit, and bank features.  We train a model to predict risk using fraud data and a model to predict risk using credit data.  Then, we use the outputs of the fraud and credit models as scores alongside additional bank features to predict the final customer risk.  (b) Ablation tests (ablating top 5 positive/negative features out of a total 22 features) comparing model agnostic approaches (LIME, KernelSHAP, IME with 100 samples), which require access to all models in the pipeline, and DeepSHAP, which allows institutions to keep their proprietary models private. (c) Summary plot of the top six features the bank model uses to predict risk (TreeSHAP).  (d) Summary plot of the top six features the entire pipeline uses to explain risk (DeepSHAP).  The green features originate from the fraud data, and the yellow features from the credit data.  We explain 1000 randomly samples explicands using 100 randomly sampled baselines for all attribution methods.  Note that (c) and (d) show summary plots (Appendix Section \ref{sec:supp:summary_plot}).}
\label{fig:model_stack}
\end{figure}

We evaluate DeepSHAP explanations for a \textbf{consumer scoring} example that feeds a simulated GBT fraud score model and a simulated MLP credit score model into a GBT bank model, which classifies good risk performance (0.681 test ROCAUC) (Figure \ref{fig:model_stack}).  Consumer scores (e.g., credit scores, fraud scores, health risk scores, etc.) describe individual behavior with predictive models \cite{dixon2014scoring}.  A vast industry of data brokers generates consumer scores based on a plethora of consumer data.  For instance, a single data broker in a 2014 FTC study had 3000 data segments on nearly every consumer in the United States, and another broker added three billion new records to its databases each month \cite{schmitz2014secret}.\footnote{Furthermore, even the HELOC data set that we use initially had an ``ExternalRiskEstimate'' feature that we removed for clarity.}  Unfortunately, explaining the models that use consumer scores can obscure important features.  For instance, explaining the bank model in Figure \ref{fig:model_stack}a will tell us that fraud and credit scores are important (in Figure \ref{fig:model_stack}c), but these scores are inherently opaque to consumers \cite{dixon2014scoring}.  The truly important features may instead be those that these scores use.  A better solution might be model-agnostic methods that explain the entire pipeline at once.  However, the model-agnostic approaches require access to all models.  In Figure \ref{fig:model_stack}a, a single institution would have to obtain access to fraud, credit, and bank models to use the standard model-agnostic approaches (Figure \ref{fig:model_stack}b (left)).  This may be fundamentally impractical because each of these models is proprietary.  This opacity is concerning given the growing desire for transparency in artificial intelligence \cite{goodman2017european,dixon2014scoring,schmitz2014secret}.  

DeepSHAP naturally addresses this obstacle by enabling attributions to the original features without forcing companies to share their proprietary models \textit{if} each institution in the pipeline agrees to work together and has a consistent set of baselines.  Furthermore, DeepSHAP can combine any other efficiency-satisfying feature attribution method in an analogous way (e.g., integrated/expected gradients \cite{sundararajan2017axiomatic}).  Altogether, DeepSHAP constitutes an effective way to ``glue'' together explanations across distributed models in industry.   In particular, in Figure \ref{fig:model_stack}a, the lending institution can explain its bank model in terms of bank features and fraud and credit scores.  The bank then sends fraud and credit score attributions to their respective companies, who can use them to generate DeepSHAP attributions to the original fraud and credit features.  The fraud and credit institutions then send the attributions back to the bank, which can provide explanations in terms of the original, more interpretable features to their applicants (Figure \ref{fig:model_stack}d).

We first quantitatively verify that the DeepSHAP attributions for this pipeline are comparable to the model agnostic approaches in Figure \ref{fig:model_stack}b.  We once again see that DeepSHAP attributions are competitive with the best performing attributions methods for ablating the top 5 most important positive or negative features.  Furthermore, we see that DeepSHAP is several orders of magnitude faster than the best performing ablation methods (KernelSHAP and IME) and an order of magnitude faster and much more performant than LIME. 

We can qualitatively verify the attributions in Figures \ref{fig:model_stack}c-d.  In Figure \ref{fig:model_stack}c, we find that the fraud and credit scores are extremely important to the final prediction.  In addition, bank features including low revolving balance divided by credit limit (``NetFractionRevolvingBurden'') and low number of months since inquisitions (``MSinceMostRecentInqExcl7Days'') are congruously important to good risk performance.  Then, in Figure \ref{fig:model_stack}d we use the generalized rescale rule to obtain attributions in the original feature space.  Doing so uncovers important variables hidden by the fraud and credit scores.  In particular, we see that the fraud score heavily relied on a high number of months since the applicants's oldest trade (``MSinceOldestTradeOpen''), and the credit score relied on a low number of months since recent delinquency (``MSinceMostRecentDelq'') in order to identify applicants that likely had good risk performance.  Importantly, the pipeline we analyze in Figure \ref{fig:model_stack}a also constitutes a stacked generalization ensemble, which we analyze more generally in Appendix Section \ref{sec:supp:stacked_gen}.

\section{Discussion}

In this manuscript, we presented examples where explaining a series of models is critical.  Series of models are prevalent in a variety of applications (health, finance, environmental science, etc.), where understanding model behavior contributes important insights.  Furthermore, having a fast approach to explain these complex pipelines may be a major desiderata for a diagnostic tool to debug ML models.  

The practical applications we focus on in this paper include gene set attribution, where the number of features far surpasses the number of samples.  In this case, we provide a rule that aggregates group attributions to higher level groups of features while maintaining efficiency.  Second, we demonstrate the utility of explaining transformations of a model's default output (Appendix Section \ref{sec:supp:nhanes_prob_vs_logodds}).  Explaining the probability output rather than the log-odds output of a logistic model yields more naturally interpretable feature attributions.  Furthermore, explaining the loss of a logistic model enables debugging model performance and identification of covariate shift.  A third application is neural network feature extraction, where pipelines may include transformations of the original features fed into a different model.  In this setting we demonstrate the computational tractability of DeepSHAP compared to model-agnostic approaches.  Finally, because our approach propagates feature attributions through a series of models while satisfying efficiency at each step (Methods Section \ref{sec:methods:efficiency}), the intermediary attributions at each part of the network can be interpreted as well.  We use this to understand the importance of both consumer scores and the original features used by the consumer scores.

In consumer scoring, distributed proprietary models (i.e., models that exist in different institutions) have historically been an obstacle to transparency.  This lack of transparency is particularly concerning given the prevalence of consumer scores, with some data brokers having thousands of data segments on nearly every American consumer \cite{schmitz2014secret}.  In addition, many new consumer scores fall outside the scope of previous regulations (e.g., the Fair Credit Reporting Act and the Equal Credit Opportunity Act) \cite{dixon2014scoring}.  In fact, these new consumer scores that depend on features correlated with protected factors (e.g., race) can reintroduce discrimination hidden behind proprietary models, which is an issue that has historically been a concern in credit scores (the oldest existing example of a consumer score) \cite{dixon2014scoring}.  DeepSHAP naturally enables feature attributions in this setting and takes a significant and practical step towards increasing the transparency of consumer scores and provides a tool to help safeguard against hidden discrimination.

It should be noted that we focus specifically on evaluating DeepSHAP for a series of mixed model types.  Previous work evaluates the rescale rule for explaining deep models, specifically.  The original presentation of the rescale rule \cite{deeplift} demonstrates its applicability to deep networks in explaining digit classification and regulatory DNA classification.  \citeauthor{schwab2019cxplain} shows that for explaining deep networks, DeepSHAP is a very fast yet performant approach in terms of an ablation test for explaining MNIST and CIFAR images.\footnote{Although their approach, CXPlain, is comparably fast at attribution time, it has the added cost of training a separate explanation model.}   Finally, \citeauthor{sixt2019explanations} shows that many modified back propagation feature attribution techniques are independent of the parameters of later layers, with the exception of DeepLIFT.  This particularly significant finding suggests that compared to most fast back propagation-based deep feature attribution approaches, DeepSHAP (which relies on the same rescale rule as DeepLIFT) is not ignorant of later layers in the network.

Although DeepSHAP works very well for explaining a series of mixed model types in practice, an inherent limitation is that it is not guaranteed to satisfy the desirable axioms (e.g., implementation invariance) that other feature attribution approaches satisfy (assuming exact solutions to their intractable problem formulations) \cite{strumbelj2010efficient,lundberg2017unified,sundararajan2017axiomatic}.  This suggests that DeepSHAP may be more appropriate for model debugging or for identifying scientific insights that warrant deeper investigation, particularly in settings where models or the input dimension is huge and tractability is a major concern. However, for applications where high-stakes decision making is important, it may be more appropriate to run axiomatic approaches to completion or use interpretable models \cite{rudin2019stop}. Furthermore, in many real world circumstances, such as distributed proprietary models based on credit risk scores, exact axiomatic approaches and interpretable models are not feasible. In these cases DeepSHAP represents a promising direction that allows multiple agents to collaboratively build explanations while maintaining separation of model ownership.

\section{Methods}
\label{sec:methods}

\subsection{The Shapley value}
\label{sec:methods:shapley}

The Shapley value is a solution concept for allocating credit among players ($M={1,\cdots,m}$) in an $m$-person game.  The game is fully described by a set function $v(S):\mathcal{P}(S)\to \mathbb{R}^1$ that maps the power set of players $S\subseteq M$ to a scalar value.  The Shapley value for player $i$ is the average marginal contribution of that player for all possible permutations of remaining players:
\begin{equation}
\phi_i(v) = \frac{1}{m!} \sum_{P\in \Sigma_{m-1}(M)} (v(S^P\cup i) - v(S^P)).
\end{equation}
We denote the finite symmetric group $\Sigma_{n-1}(M)$, which is the set of all possible permutations, and $S^P$ to be the set of players before player $i$ in the permutation $P$.  The Shapley value is a provably unique solution under a set of axioms (Appendix Section \ref{sec:supp:shapley_axioms}).  One axiom that we focus on in this paper is \textit{efficiency}:
\begin{equation}
\sum_{i=1}^m \phi_i(v) = v(M) - v(\emptyset).
\end{equation}

\subsubsection*{Adapting the Shapley value for feature attribution of ML models}

Unfortunately, the Shapley value cannot assign credit for an ML model ($f(x):\mathbb{R}^m\to \mathbb{R}^1$) directly because most models require inputs with values for every feature, rather than a subset of features.  Accordingly, feature attribution approaches based on the Shapley value define a new set function $v(S)$ that is a ``lift'' of the original model \cite{merrill2019generalized}.  In this paper, we focus on local feature attributions,\footnote{As opposed to global feature attributions, which measure feature importance of a model across an entire data set.} which describe a model's behavior for a single sample, called an explicand ($x^e$).  A ``lift'' is defined as:
\begin{equation}
\mu(f,x^e,S):\mathbb{R}^m\times 2^m\to \mathbb{R}^1.
\end{equation}

% \subsection{Common lifts of machine learning models}

One common lift is the \textit{observational conditional expectation}, where the lift is the conditional expectation of the model's output holding features in $S$ fixed to $x^e_S$ and $X$ is a multivariate random variable with joint distribution $D$: 
\begin{equation}
    \mu^{obs}_D(f,x^e,S)=\mathbb{E}_D[f(X)|X_S=x^e_S].
\end{equation}
  
Another common lift is the \textit{interventional conditional expectation} with a flat causal graph, where we "intervene" on features by breaking the dependence between features in $X_S$ and the remaining features using the causal inference $do$-operator \cite{janzing2019feature}:
\begin{equation}
    \mu^{int}_D(f,x^e,S)=\mathbb{E}_D[f(X)|do(X_S=x^e_S)].
    \label{eq:int_lift}
\end{equation}

Both approaches have tradeoffs that have been described elsewhere \cite{chen2020true,kumar2020problems,sundararajan2020themanyshapleyvalues,merrick2020explanation,frye2020shapley}.  Here, we focus on the interventional approach because it is most closely related to DeepSHAP and does not require estimating the joint density of $X$.

The Shapley values computed for any lift will satisfy efficiency in terms of the lift $\mu$.  However, for the interventional and observational lift described above, the Shapley value will also satisfy efficiency in terms of the model's prediction:
\begin{equation}
    \sum_i\phi^{\mu_D}_i(f,x^e)=f(x^e)-\mathbb{E}_D[f(X)].
\end{equation}
This means that attributions can naturally be understood to be in the scale of the model's predictions (e.g., log-odds or probability for binary classification).

\subsection{Interventional Shapley values baseline distribution}
\label{sec:methods:baseline}

We can define a single baseline lift
\begin{equation}
\mu^{int}_{x^b}(f,x^e,S)=\mathbb{E}_{\{x^b\}}[f(X)|do(X_S=x^e_S)]=\chi^{S},
\end{equation}
where $\chi^{S}$ is a spliced sample and $\chi^{S}_i=x^e_i$ if $i\in S$, else $\chi^{S}_i=x^b_i$.

Then, we can decompose the Shapley value $\phi_i(f,x^e)$ for the interventional conditional expectation lift (eq. \ref{eq:int_lift}) (henceforth referred to as the interventional Shapley value) into an average of Shapley values with single baseline lifts (proof in Appendix Section \ref{sec:supp:int_baseline_dist_proof}):
\begin{align}
\phi_i(f,x^e,D)&=\frac{1}{|D|}\sum_{x^b\in D} \underbrace{\frac{1}{m!}\sum_{P\in \Sigma_{m-1}(M)} f(\chi^{S^P\cup i}) {-} f(\chi^{S^P})}_\text{Shapley value for single baseline lift}\\
&= \frac{1}{|D|}\sum_{x^b\in D} \phi_i(f,x^e,x^b).
\end{align}
Here, $D$ is an empirical distribution with equal probability for each sample in a baseline data set.  An analogous result exists for the observational conditional distribution lift using an input distribution\cite{merrick2020explanation}.\footnote{The attributions for these single baseline games are also analogous to baseline Shapley in \cite{sundararajan2020themanyshapleyvalues}.} 

In the original DeepLIFT paper, \cite{deeplift} recommend two heuristic approaches to define baseline distributions: (1) choosing a sensible single baseline and (2) averaging over multiple baselines.  In addition, DeepSHAP, as previously described in \cite{lundberg2017unified}, created attributions with respect to a single baseline equal to the expected value of the inputs.  In this paper, we show that from the perspective of Shapley values with an interventional conditional expectation lift, averaging over feature attributions computed with single baselines drawn from an empirical distribution is the correct approach.  One exception to this are linear models, where taking the average as the baseline is equivalent to averaging over many single baseline feature attributions \cite{chen2020true}.  Interventional Shapley values computed with a single baseline satisfy efficiency in terms of the model's prediction:
\begin{equation}
\sum_i\phi_i(f,x^e,x^b)=f(x^e)-f(x^b).
\end{equation}

\subsection{Selecting a baseline distribution}
\label{sec:methods:baseline_kmeans}

As in the previous section, we define a baseline distribution $D$ over which we compute Shapley values with single baseline lifts.  This baseline distribution is naturally chosen to be a distribution over the training data $X^{train}$, where each sample $x^j\in \mathbb{R}^m$ has equal probability.  The interpretation of this distribution is that the explicand is compared to each baseline in $D$.  This means that the interventional Shapley values implicitly create attributions that explain the model's output relative to a baseline distribution.  

Although the entire training distribution is a natural and interpretable choice of baseline distribution, it may be desirable to use others.  To automate the process of choosing such an interpretable baseline distribution, we turn to unsupervised clustering.  We utilize k-means clustering on a reduced version of the training data ($\hat{X}^{train}$) comprised of $\hat{x}^j= [ x^j_i \forall i\in M_r ]$ with a reduced set of features ($M_r$).  The output of the k-means clustering are clusters $C_1,\cdots, C_k$ with means $\mu_1,\cdots,\mu_k$ that minimize the following objective on the reduced training data:
\begin{equation}
\argmin_{C_1,\cdots,C_k} \sum_{i=1}^k\sum_{\hat{x}\in C_i} ||\hat{x}-\mu_i||^2.
\end{equation}

Then, the cluster selected as a baseline distribution explaining an explicand $x^e$ is chosen based on:
\begin{equation}
\argmin_{i} ||\hat{x}^e-\mu_i||^2.
\end{equation}

\subsection{A generalized rescale rule to explain a series of models}
\label{sec:methods:series_of_models}

We define a \textit{generalized rescale rule} to explain an arbitrary series of models that propagates approximate Shapley values with an interventional conditional expectation lift for each model in the series.\footnote{This generalized chain rule will also generalize to any feature attribution method that satisfies efficiency.}  To describe the approach, we define a \textit{series of models} to be a composition of functions $f_k(x)=(h_k\circ \cdots \circ h_1)(x)$, and we define intermediary models $f_i(x)=(h_i \circ \cdots \circ h_1)(x),\text{ }i=1,\cdots,k$.  We define the domain and codomain of each model in the series as $h_i(x):\mathbb{R}^{m_i}\to \mathbb{R}^{o_i}$.  Then, we can define the propagation for a single baseline\footnote{The baseline distribution attribution $\phi_i(f,x^e,D)$ is then simply the average across many of these single baseline attributions $\phi_i(f,x^e,x^b)$.} recursively:
\begin{align}
\psi^k&=\hat{\phi}(h_k,x^e,x^b)\\
\psi^i&=\hat{\phi}(h_i,x^e,x^b)\big(\psi^{i+1}\oslash(f_i(x^e)-f_i(x^b))\big),\text{ }i\in 1,\cdots, k-1.
\end{align}

We use Hadamard division to denote an element-wise division of $\Vec{a}$ by $\Vec{b}$ that accommodates zero division, where if the denominator $b_i$ is 0, we set $a_i/b_i$ to 0.  Additionally, $\hat{\phi}$ are an appropriate feature attribution technique that approximates interventional Shapley values while crucially satisfying efficiency for the model $h_i$ it is explaining.  In this paper, we utilize DeepLIFT (rescale) for deep models, TreeSHAP for tree models, and exact interventional Shapley values for linear models.  We define efficiency as $\hat{1}_{1\times m_i}\hat{\phi}(h_i,x^e,x^b)=f_i(x^e)-f_i(x^b)$ where $\hat{1}_{a\times b}$ is a matrix of ones with shape $a\times b$ and the approximate Shapley value functions $\hat{\phi}$ return matrices in $\mathbb{R}^{(m_i\times o_i)}$.  The final attributions in the original feature space are:
\begin{equation}
\phi_i(f_k,x^e,x^b)=\psi^1_i.
\end{equation}

Furthermore, this approach yields intermediate attributions that serve as meaningful feature attributions.  In particular, $\psi^i$ can be interpreted as the importance of the inputs to the model $(h_k\circ \cdots \circ h_i)$, where the new explicand and baseline are $(h_{i-1}\circ \cdots \circ h_1)(x^e)$ and $(h_{i-1}\circ \cdots \circ h_1)(x^b)$, respectively.  This approach takes inspiration from the chain rule applied specifically for deep networks in \cite{deeplift}, but we extend it to more general classes of models. 

\subsection{Efficiency for intermediate attributions}
\label{sec:methods:efficiency}

As one might expect, each intermediate attribution $\psi^i$ satisfies efficiency:

\begin{theorem}
Each attribution $\psi^i\in \mathbb{R}^m, \forall i\in1,\cdots,k$ satisfies efficiency and sums up to $f_k(x^e)-f_k(x^b)$.
\end{theorem}

\begin{proof}
We will prove by induction that
\begin{equation}
\hat{1}_{1\times m_i} \psi^i = f_k(x^e)-f_k(x^b), \forall i \in 1,\cdots,k.
\label{eq:efficiency}
\end{equation}
\noindent
For simplicity of notation, denote $\hat\phi^i=\hat\phi(h^i,x^e,x^b)$.

\noindent
\textit{Assumption:} Each $\hat{\phi}$ satisfies efficiency
\begin{equation}
\hat{1}_{1\times m_i} \hat\phi^i=f_i(x^e)-f_i(x^b).
\end{equation}
\noindent
\textit{Base Case:} By our assumption, 
\begin{equation}
\hat{1}_{1\times m_k} \psi^{k}=f_k(x^e)-f_k(x^b).
\end{equation}

\noindent
\textit{Induction Step:}
\begin{align}
\psi^i &= \hat{\phi}\big(\psi^{i+1}\oslash(f_i(x^e)-f_i(x^b))\big)\\
\hat{1}_{1\times m_i} \psi^i &= \hat{1}_{1\times m_i} \hat{\phi}\big(\psi^{i+1}\oslash(f_i(x^e)-f_i(x^b))\big)\\
&= (f_i(x^e)-f_i(x^b))\big(\psi^{i+1}\oslash(f_i(x^e)-f_i(x^b))\big)\\
&= \hat{1}_{1\times o_i} \psi^{i+1}\\
&= \hat{1}_{1\times m_{i+1}} \psi^{i+1}\\
&= f_k(x^e)-f_k(x^b).
\end{align}

\noindent
\textit{Conclusion:} By the principle of induction, each intermediate attribution satisfies efficiency (eq. \ref{eq:efficiency}).

\end{proof}

Then, because the interventional Shapley value with a baseline distribution is the average of many single baseline attributions, it satisfies a related notion of efficiency:
\begin{align}
\sum_i \phi_i(f_k,x^e) &= \sum_i \sum_{x^b\in D} \phi_i(f_k,x^e,x^b)\\
&= \sum_{x^b\in D} \sum_i \phi_i(f_k,x^e,x^b)\\
&= \sum_{x^b\in D} f_k(x^e) - f_k(x^b)\\
&= f_k(x^e) - \frac{1}{|D|} \sum_{x^b\in D} f_k(x^b).
\end{align}
This can be naturally interpreted as the difference between the explicand's prediction and the expected value of the function across the baseline distribution.

An additional property of the generalized rescale rule is that although it is an approximation to the interventional Shapley values in the general case, if every model in the composition is linear ($h_i(x)=\beta x$), then this propagation exactly yields the interventional Shapley values (Appendix Section \ref{sec:supp:linear_series_of_models}).

\subsection{Connecting DeepLIFT's rules to the Shapley values}
\label{sec:methods:deeplift_shapley}

Now we can connect the Shapley values to DeepLIFT's Rescale and RevealCancel rules.  Both rules aim to satisfy an efficiency axiom (what they call \textit{summation to delta}) and can be connected to an interventional conditional expectation lift with a single baseline (as in Section \ref{sec:methods:baseline_kmeans}).

In fact, multi-layer perceptrons are a special case where the models in the series are non-linearities applied to linear functions.  We first represent deep models as a composition of functions $(h_1\circ \cdots \circ h_k)(x)$.  The Rescale and RevealCancel rules canonically apply to a specific class of function: $h_i(x)=(f\circ g)(x)$, where $f$ is a non-linear function and $g$ is a linear function parameterized by $\beta\in \mathbb{R}^m$.  We can interpret both rules as an approximation to interventional Shapley values based on the following definition.

\begin{definition}[k-partition approximation]
A k-partition approximation to the Shapley values splits the features in $x\in \mathbb{R}^m$ into $K$ disjoint sets.  Then, it exactly computes the Shapley value for each set and propagates it linearly to each component of the set.
\end{definition}

The Rescale rule can be described as a 1-partition approximation to the Interventional Shapley values for $h_i(x)$, while the RevealCancel rule can be described as a 2-partition approximation that splits according to whether $\beta_i x_i>t$, where the threshold $t=0$.  This k-partition approximation lets us consider alternative variants of the Rescale and RevealCancel rules that incur exponentially larger costs in terms of $K$ and for different choices of thresholds.

\subsection{Explaining groups of input features}
\label{sec:methods:groups}

Here, we further generalize the Rescale rule to support groupings of features in the input space.  Having such a method can be particularly useful when explaining models with very large numbers of features that are more understandable in higher level groups.  One natural example is gene expression data, where the numbers of features is often extremely large.  

We introduce a \textit{group rescale rule} that facilitates higher level understanding of feature attributions.  We can define a set of groups $G_1,\cdots,G_o$ whose members are the input features $x_i$.  If each group is disjoint and covers the full set of features, then a natural group attribution that satisfies efficiency is the sum:  
\begin{equation}
\phi^0_{G_j}(f,x^e)=\sum_{i\in G_j}\phi_{i}(f,x^e).
\label{eq:group_attr_0}
\end{equation}
If the groups are not disjoint or do not cover all input features, then the above attributions do not satisfy efficiency.  To address this, we define a residual group $G_R$ that covers all input features not covered by the remaining groups.  Then, the new attributions are a rescaled version of eq. \ref{eq:group_attr_0}
\begin{align}
\phi_{G_j}(f,x^e)&=\phi^0_{G_j}(f,x^e)\times \frac{\sum \phi_{G_j}(f,x^e)}{\sum \phi_{i}(f,x^e)}.
\end{align}

We can naturally extend this approach to accommodate non-uniform weighting of group elements, although we do not experiment with this in our paper.
% \hugh{Appendix - non-uniform weights}

\subsection{Ablation Tests}
\label{sec:methods:ablation_test}

We evaluate our feature attribution methods with \textit{ablation tests} \cite{hooker2019benchmark,treeshap}.  In particular, we rely on a simple yet intuitive ablation test.  For a matrix of explicands $X^e\in\mathbb{R}^{n_e,m}$, we can get attributions $\phi(f,X^e)\in\mathbb{R}^{n_e,m}$.  The ablation test is defined by three parameters: (1) the feature ordering, (2) an imputation sample $x^b\in \mathbb{R}^m$, and (3) an evaluation metric.  Then, the ablation test replaces features one at a time with the baseline's feature value based on the feature attributions to assess the impact on the evaluation metric.  We can iteratively define the ablation test based on modified versions of the original explicands: 
\begin{align}
X^{e,0} &= X^e\\
X^{e,k} &= X^e \odot I_k(\phi) + x^b \odot (1-I_k(\phi)), \forall k \in 1,\cdots, m.
\end{align}

Note that $x^b\coloneqq[\underbrace{x^b \cdots x^b}_{n_e\text{ elements}}]^T$ and $I_k(\phi)\coloneqq I_k(\phi(f,X^e))=\argmax_{k,axis=1}(\phi(f,X^e))$, where $\argmax_{k,axis=1}(G)$ returns an indicator matrix of the same size as $G$ and 1 indicates that the element was in the maximum $k$ elements across a particular axis.  

Then, the ablation test measures the mean model output (e.g., the predicted log-odds, predicted probability, the loss, etc.) if we ablate $k$ features to be the average over the predictions for each ablated explicand: 
\begin{align}
\frac{1}{n_e} \sum_{i\in 1,\cdots,n_e} f(X^{e,k}_i).
\end{align}
Note that for our ablation tests we focus on either the positive or the negative elements of $\phi$, since the expected change in model output is clear if we ablate only by positive or negative attributions.

Ablation tests are a natural approach to test whether feature attributions are correct for a set of explicands.  For feature attributions that explain the predicted log-odds, a natural choice of model output for the ablation test is the mean of the log-odds predictions.  Then, as we ablate increasing numbers of features, we expect to see the model's output change.  When we ablate the most positive features, the mean model output should decrease substantially.  As we ablate additional features, the mean model output should still decrease, but less drastically so.  This implies that, for positive ablations, lower curves imply attributions that better described the model's behavior.  In contrast, for negative ablations

\section{Acknowledgements}

We thank Ayse B. Dincer, Pascal Sturmfels, Joseph Janizek, and Gabe Erion for helpful discussions.  We thank the NSF GRFP DGE-1762114 for support.

\section{Competing Interests}

The authors declare that there are no competing interests.

\section{Data availability}

The NHANES I, NHANES 1999-2014, CIFAR, and MNIST data sets are all publicly available.  The HELOC data set can be obtained by accepting the data set usage license: (\url{https://community.fico.com/s/explainable-machine-learning-challenge?tabset-3158a=a4c37}).  Metabric data access is restricted and requires getting an approval through Sage Bionetworks Synapse website: \url{https://www.synapse.org/#!Synapse:syn1688369} and \url{https://www.synapse.org/#!Synapse:syn1688370}.  ROSMAP data access is restricted and requires getting an approval through Sage Bionetworks Synapse website: \url{https://www.synapse.org/#!Synapse:syn3219045} and is available as part of the AD Knowledge Portal \cite{greenwood2020ad}.

\section{Code availability}

The code for the experiments is available here: \url{https://github.com/suinleelab/DeepSHAP}.

\section{Author Contribution}

H.C. contributed to study design, data analysis, and manuscript preparation. S.M.L. contributed to data analysis and manuscript preparation. S.L. contributed to study design, and manuscript preparation.

\newpage

% \bibliography{main}
% \bibliographystyle{plainnat}

\vskip 0.2in
\printbibliography

\appendix

\section{Appendix}

\subsection{Data Sets}

\subsubsection{NHANES I}
\label{sec:supp:nhanes_i}

The National Health and Nutrition Examination Survey (NHANES) I \cite{cox1998plan} is a national longitudinal study conducted on a random sample of individuals from the United States.  NHANES I investigates a number of demographics and socioeconomic variables.  We utilize the NHANES I Epidemiologic Follow-up Study (NHEFS) which is designed to investigated the relationships between clinical, nutritional, and behavioral factors originally assessed in NHANES I.  The NHEFS study comprised a series of follow up studies that trace the cohort (all persons 25-74 years of age who completed a medical examination in NHANES I) and measure additional variables as well as collect death certificates.

\subsubsection{NHANES 1999-2014}
\label{sec:supp:nhanes_1999_2014}

The National Health and Nutrition Examination Survey (NHANES) continually collects information on subsamples of the civilian noninstitutionalized US population in two-year cycles.  We collected the data from these cycles from 1999-2014 yielding a total of eight release cycles.  The surveys collect a variety of laboratory, questionnaire, examination, and demographic data.  In particular, the features collected do not match across cycles, so we only utilize variables that are consistently collected across cycles.  

\subsubsection{ROSMAP Alzheimer's Gene Expression}
\label{sec:supp:alzheimers}

Gene expression data collected from the Religious Orders Study (ROS) and Memory and Aging Project (MAP) \cite{a2012overview,bennett2018religious}.  ROS is a longitudinal cohort study of aging and Alzheimer's disease run by Rush University enrolling individuals from religious communities for longitudinal clinical analysis and brain donation.  MAP is a longitudinal epidemiologic cohort study of common chronic conditions of aging run by Rush University that aims to complement the ROS study by enrolling individuals with wider life experiences and socioeconomic status.  Both studies aim to study aging and risk of Alzheimer's disease.  We utilize gene expression data collected using ChIP-seq and predict Alzheimer's disease status of the corresponding patients.

% https://www.radc.rush.edu/resources/docs/RADC%20Resource%20Sharing%20Policies%20082517.pdf

\subsubsection{METABRIC Breast Cancer Gene Expression}
\label{sec:supp:breast_cancer}

Gene expression data collected from Molecular Taxonomy of Breast Cancer International Consortium (METABRIC) \cite{curtis2012genomic,pereira2016somatic}.  METABRIC analyzes genomic and transcriptomic information from a set of 995 breast cancer tumors.  Although the original analyses of the transcriptomic information from the tumors are used for a variety of analyses, we purely utilize the transcriptomic information to predict tumor status.

% https://journals.plos.org/ploscompbiol/article?id=10.1371/journal.pcbi.1004888

\subsubsection{CIFAR}
\label{sec:supp:cifar}

The CIFAR10 data set consists of $32\times 32$ color images with 10 possible classes that are a labeled subset of the 80 million tiny images data set \cite{krizhevsky2009learning}.  The mutually classes include airplanes, automobiles, birds, cats, deers, dogs, frogs, horses, ships, and trucks.  In particular, the images were collected by colleagues at MIT and NYU with natural images collected on a number of search engines.

\subsubsection{MNIST}
\label{sec:supp:mnist}

The MNIST database consists of $28\times 28$ black and white handwritten digits \cite{lecun1998mnist}.  The digits are size-normalized and centered in a fixed-size image.  There are ten possible classes that correspond to the digits $0,\cdots,9$.

\subsubsection{HELOC}
\label{sec:supp:heloc}

The Home Equity Line of Credit (HELOC) data set \cite{heloc} is an anonymized data set of HELOC applications from real homeowners. A HELOC is a line of credit offered by a bank as a percentage of home equity. The outcome is whether the applicant will repay their HELOC account within two years.  Financial institutions use predictions of loan repayment to decide whether applicants qualify for a line of credit. The data set was released as part of a FICO xML Challenge (\url{https://community.fico.com/s/explainable-machine-learning-challenge}) and can be obtained under appropriate agreement to a data set usage license.

% \subsection{Ethics}

% The NHANES I, NHANES 1999-2014, CIFAR, and MNIST data sets are all publicly available.  We obtained the HELOC data set by accepting the data set usage license: (\url{https://community.fico.com/s/explainable-machine-learning-challenge?tabset-3158a=a4c37}).  

\subsection{Experimental Setup}
\label{sec:supp:exp_setup}

\subsubsection{CIFAR Multiple Baseline}
\label{sec:supp:mult_ref_cifar_setup}

\textbf{Model Hyperparameters:} The model explained is a CNN with the following sequence of layers: a convolutional layer with 32 filters of shape 3 by 3 with a ReLU activation, a convolutional layer with 32 filters of shape 3 by 3 with a ReLU activation, a max pooling layer with of size 2 by 2, a dropout layer with 0.25 probability, a convolutional layer with 64 filters of shape 3 by 3 with a ReLU activation, a convolutional layer with 64 filters of shape 3 by 3 with a ReLU activation, a max pooling layer with of size 2 by 2, a dropout layer with 0.25 probability, a dense layer with 512 nodes with ReLU activation, a dropout layer with probability 0.5, a dense output layer with softmax activation.  RMSprop with a learning rate of 0.0001 and decay of $1\times10^{-6}$ is used to optimize the network for a categorical cross entropy loss over 100 epochs with batch sizes of 32.  The test accuracy achieved by the model is 75.56\%.

\textbf{Experimental setup:} In Figure \ref{fig:mult_ref_cifar} we explain three explicands with black objects: a plane, a horse, and an ostrich.  For the single baseline attributions we utilize DeepSHAP with a single black image as the baseline.  For multiple baselines we utilize DeepSHAP with a baseline distribution of 1000 randomly sampled images from the training data set.  The feature attribution plots take the local feature attributions for the softmax output corresponding to the true label.  For simple visualization we take the absolute value of the attributions and average across channels to get a grayscale image which we plot after normalizing the attribution values between zero and one.  The pixel distributions are the number of pixels in the gray scale version of the explicand image that fell within ten equally sized gray scale value bins.  The attribution distribution is the sum of the attribution mass for the pixels in the original image that correspond to each gray scale value bin. 

\subsubsection{NHANES Multiple Baseline}
\label{sec:supp:mult_ref_nhanes_setup}

\textbf{Model Hyperparameters:} The model we explain is an MLP with four hidden layers with 100 nodes each.  The hidden layers have ReLU activation functions and dropout layers in between.  The final output node is a sigmoid activation trained to minimize binary cross entropy loss to optimize mortality classification.  RMSprop with a learning rate of 0.001 is used to optimize the network over 50 epochs with batch sizes of 128.  The test ROC achieved by the model is .872.  

\textbf{Experimental setup:} In Figure \ref{fig:mult_ref_nhanes}a, the explicand is a randomly chosen older male individual from the NHANES data set.  In the top force plot we show DeepSHAP attributions for the explicand with a baseline distribution of 1000 randomly chosen samples from the training set.  In the bottom force plot we show DeepSHAP attributions for the same explicand with a baseline distribution of 1000 randomly chosen samples from older (>60 years old) males from the training set.  In Figure \ref{fig:mult_ref_nhanes}b, the clusters are obtained by k-means clustering (k=8) the training data with only two features: age and sex.  In Figure \ref{fig:mult_ref_nhanes}c, the explicands are the older male cluster in the training data (n=1137).  We show two summary plots where the top are DeepSHAP attributions with a baseline distribution of 1000 randomly chosen samples from the training set and the bottom uses the older male cluster as a baseline distribution.  In Figure \ref{fig:mult_ref_nhanes}d, we perform an ablation test that ablates all explicands in the older male clusters according to either their most positive or most negative local feature attributions.  When ablating we impute by the mean feature value in the older male cluster.   Then we evaluate the model's prediction across all of the explicands after ablating features one at a time.  

\subsubsection{Gene set explanations}
\label{sec:supp:pathway_setup}

\textbf{Model hyperparameters:} We train two GBToost classifiers (an implementation of gradient boosting trees) to predict our binary phenotypes (Alzheimer's and breast cancer tumor stage) based on transcriptomic data.  The classifiers are trained with a learning rate of 0.3, a max tree depth of 6, and automatic heuristic tree construction.  We train with a validation set and 10 early stopping rounds.  For Alzheimer's classification we achieve a test ROC AUC of 0.959 and for breast cancer tumor stage classification we achieve a test ROC AUC of 0.932.

\textbf{Experimental setup:} The feature attributions for the tree model are obtained using Interventional Tree Explainer \cite{treeshap}.  These attributions correspond to the importance of each gene to the log odds of the output phenotypes.  In order to explain these attributions in terms of groups we utilize our group rescale rule to propagate the gene attributions to pathway attributions.  For Alzheimer's we fix the baseline distribution to be the training data set and for breast cancer which has more samples we fix a baseline distribution of 100 random samples from the training set for breast cancer.

\subsubsection{NHANES Loss explanations}
\label{sec:supp:loss_setup}

\textbf{Model hyperparameters:} We train an GBToost classifier (an implementation of gradient boosting trees) to predict our mortality based on epidemiological features.  The classifier is trained with a learning rate of 0.3, a max tree depth of 6, and automatic heuristic tree construction.  We train with a validation set and 10 early stopping rounds.  For the weight-shifted test set we achieve an ROC AUC of 0.860 and for the non-shifted test set we achieve an ROC AUC of 0.868.

\textbf{Experimental setup:} In Figure \ref{fig:loss_attributions}b, we generate output feature attributions using Interventional Tree Explainer and use DeepSHAP's generalized rescale rule to explain the loss in addition to the output.  The loss and output attributions are explained with respect to the same baseline distribution of 1000 random samples from the training set.  The loss attributions for positive labelled explicands and negative labelled explicands are very different, leading us to plot them as separate dependence plots.   In Figure \ref{fig:loss_attributions}c, we generate output and loss feature attributions as before.  For the ablation, we do a simplified univariate ablation where we impute the blood loss to the mean of the baseline distribution for samples selected based on largest loss attributions.  In Figure \ref{fig:loss_attributions}d, we perform an ablation test that ablates 1000 explicands from the training set according to either their loss or output local feature attributions.  When ablating we impute by the mean value of a given feature in the explicands.   Then we evaluate the model's prediction across all of the explicands after ablating features one at a time.

\subsubsection{MNIST Feature Extraction}
\label{sec:supp:feature_extraction_setup}

\textbf{Model hyperparameters:} We train a CNN model to classify all digits in MNIST.  The CNN model consists of a convolutional layer with 32 filters of size 3 by 3 with ReLU activation, a max pooling layer with pools of size 2 by 2, a convolutional layer with 64 filters of size 3 by 3 with ReLU activation, a max pooling layer with pools of size 2 by 2, a dense layer with 100 nodes and ReLU activation, and the dense output layer with 10 nodes and softmax activation.  We utilize categorical cross-entropy loss, an Adam optimizer with learning rate 0.001, and train for 10 epochs.  Then, in order to utilize the model to to extract higher level features from raw MNIST images, we remove the final output layer.  The GBToost model we train to predict zeros using the MNIST features has a max tree depth of 5, a learning rate of 0.5, and a binary logistic objective.  This model achieves a test accuracy of 0.998 for predicting zeros.

\textbf{Experimental setup:} In this experiment, we train a CNN model and use it to extract features that are fed into an GBT model.  In Figure \ref{fig:feature_extraction}a, we show the feature attributions for DeepSHAP and three model-agnostic approaches.  Each model-agnostic approach uses a number of samples which is set to 100,000.  All models utilize the same baseline distribution of 100 random images to explain the five images we selected.  In Figure \ref{fig:feature_extraction}b we report the runtime of these feature attribution approaches, and the ablation of the top 10\% of features.  In order to ablate the top ten positive (or negative) features, we simply select the pixels with the largest positive (or negative) attribution in the five explicands and impute them with the mean pixels across the baseline distribution.  We obtain confidence intervals by repeating this 20 times for different randomly selected sets of five explicands, where we enforce that at least one zero occurs within the five explicands, because it is the class of interest.

\subsubsection{HELOC Stacked Generalization}
\label{sec:supp:stacked_generalization_setup}

\textbf{Model hyperparameters:} In this experiment we train two base-learners.  One base learner is a GBT classifier that represents a fraud detection model which utilizes the following features: ``MSinceOldestTradeOpen'', ``MSinceMostRecentTradeOpen'', and ``NumTradesOpeninLast12M''.  Although this classifer represents a fraud detection model, we train it to predict risk using a learning rate of 0.1, 100 estimators, and a max tree depth of 3.  The other base learner is that represents a credit scoring model which utilizes the following features: ``AverageMInFile'', ``NumSatisfactoryTrades'', ``NumTrades60Ever2DerogPubRec'', ``NumTrades90Ever2DerogPubRec'', ``PercentTradesNeverDelq'', ``MSinceMostRecentDelq'', ``MaxDelq2PublicRecLast12M'', ``MaxDelqEver'', and ``NumTotalTrades''.  We train the base learner to predict risk using an MLP consisting of two hidden layers with 100 nodes and ReLU activations and an output layer consisting of a single dense node with sigmoid activation.  The binary cross-entropy loss function is optimized using stochastic gradient descent and a learning rate of 0.005.  The meta learner is a GBT classifier that represents a bank risk prediction model which utilizes the remaining HELOC features in addition to the outputs of the two base learners.  The meta learner uses the following hyperparameters: learning rate of 0.1, 100 estimators, and a max tree depth of 3.

\textbf{Experimental setup:}
In Figure \ref{fig:model_stack}a, we first train a GBT and MLP base-learner on disjoint subsets of features from the training data.  Then we generate scores for the training data and append it to the remaining features.  The remaining features and consumer scores are used to train a final GBT model.  Finally, we evaluate the final GBT on a held out test data set.  In Figure \ref{fig:model_stack}b, we create explanations for the meta model using interventional Tree Explainer for the GBT.  Then in \ref{fig:model_stack}c, we use the generalized rescale rule to propagate the attributions back through the base-learners (GBT and MLP) to obtain attributions in the original feature space.

\subsubsection{NHANES Stacked Generalization}
\label{sec:supp:nhanes_stacked_generalization_setup}

\textbf{Model hyperparameters:} In this experiment we train five base-learners - MLPs.  The MLPs consist of two hidden layers with 100 nodes and ReLU activations.  The output layer is a single dense node with sigmoid activation.  The binary cross-entropy loss function is optimized via stochastic gradient descent with a learning rate of 0.005.  Then we train a two meta-models that use the outputs of the MLPs as inputs.  The first is a logistic regression model with an L2 penalty and regularization strength of 1.  The second is a gradient boosted trees classifier with a learning rate of 0.1, 100 estimators, and a max tree depth of 3.  

\textbf{Experimental setup:}
In Figure \ref{fig:supp:stacked_gen}a, we first train five MLP base-learners on training data.  Then we embed held out validation data using the predictions of the five MLP base-learners.  This embedded validation data is used to train the logistic regression and gradient boosting trees models.  Finally, all models are evaluated on a held out test data set.  In Figure \ref{fig:supp:stacked_gen}b, we create meta-level explanations using interventional Shapley value attributions for the linear models (average voting and logistic regression) \cite{chen2020true}, and interventional Tree Explainer for the GBT.  Then in \ref{fig:supp:stacked_gen}c, we use the generalized rescale rule to propagate the attributions back through the base-learner MLPs to obtain attributions in the original feature space.

\subsection{Feature attribution plots}
\label{sec:supp:feat_attr_plots}

In this section we describe a number of plotting techniques for conveying information about local feature attributions.  These plots were first introduced in \cite{treeshap}.

\subsubsection{Force plots}
\label{sec:supp:force_plot}

Force plots show the feature attributions for a single explicand in terms of how they drive the model's prediction for the explicand away from the average model prediction across the baseline distribution.  The width of the bars indicate the feature attribution value with red indicating a positive affect and blue indicating a negative one.  The features corresponding to the largest bars are below with their actual values for the explicand.

\subsubsection{Dependence plots}
\label{sec:supp:dependence_plot}

Dependence plots show the feature attributions for many explicands for a single feature.  Every point corresponds to a single explicand where the x-axis is the value of the feature and the y-axis is the the feature attribution value.  The coloring of the points often denotes the value of a separate feature.

\subsubsection{Summary plots}
\label{sec:supp:summary_plot}

Summary plots show the feature attributions for many explicands and multiple features.  Summary plots stack multiple subplots plots for each individual feature.  For the feature plots, every point corresponds to a single explicand where the x-axis is the feature attribution value and the y-axis is vertical dispersion representing the frequency of samples with a particular feature attribution value.  Finally, the color of each point represents the normalized feature value, with red representing a high value and blue representing a low one.  Intermediary feature values are interpolations between red and blue.

\subsection{Results}

% \subsubsection{Ablation tests}

% \hugh{Add the simulated data ablation tests}

\subsubsection{Additional CIFAR bias examples}
\label{sec:supp:mult_ref_cifar_ig_ime}

We present additional examples of bias for IME and integrated gradients in Figures \ref{fig:supp:bias_cifar_ime} and \ref{fig:supp:bias_cifar_intgrad}.

\begin{figure}[!ht]
\includegraphics[width=.8\textwidth]{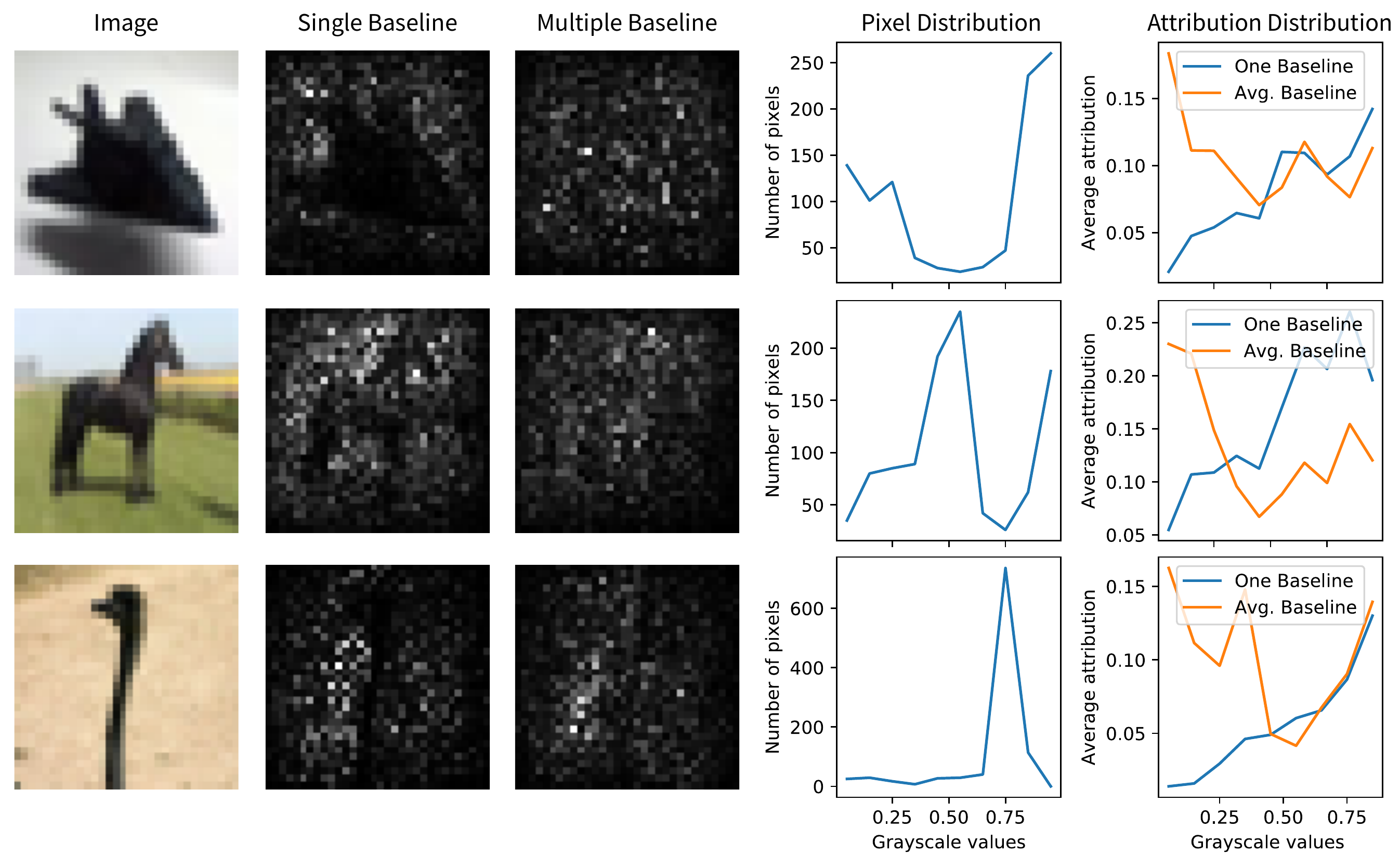}
\centering
\caption{Demonstrating bias of a single baseline for IME.}
\label{fig:supp:bias_cifar_ime}
\end{figure}

\begin{figure}[!ht]
\includegraphics[width=.8\textwidth]{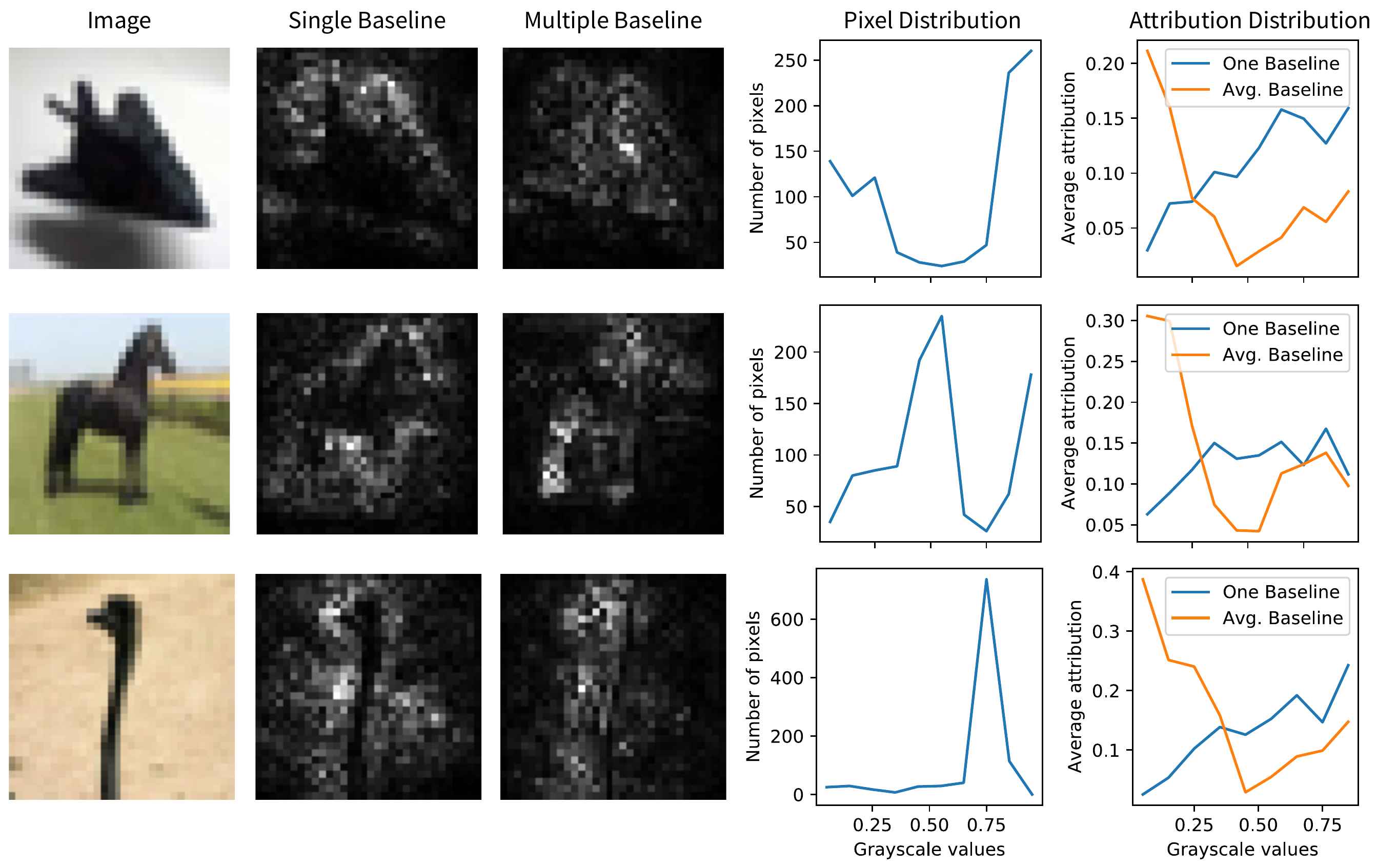}
\centering
\caption{Demonstrating bias of a single baseline for integrated/expected gradients.}
\label{fig:supp:bias_cifar_intgrad}
\end{figure}

\subsubsection{Probability vs. log-odds explanations}
\label{sec:supp:nhanes_prob_vs_logodds}

In Figure \ref{fig:supp:summary_prob_lodds} we illustrate the difference between explanations in log-odds versus probability space using attributions obtained from rescaling the log-odds explanations provided by TreeSHAP.

\begin{figure}[!ht]
\includegraphics[width=.7\textwidth]{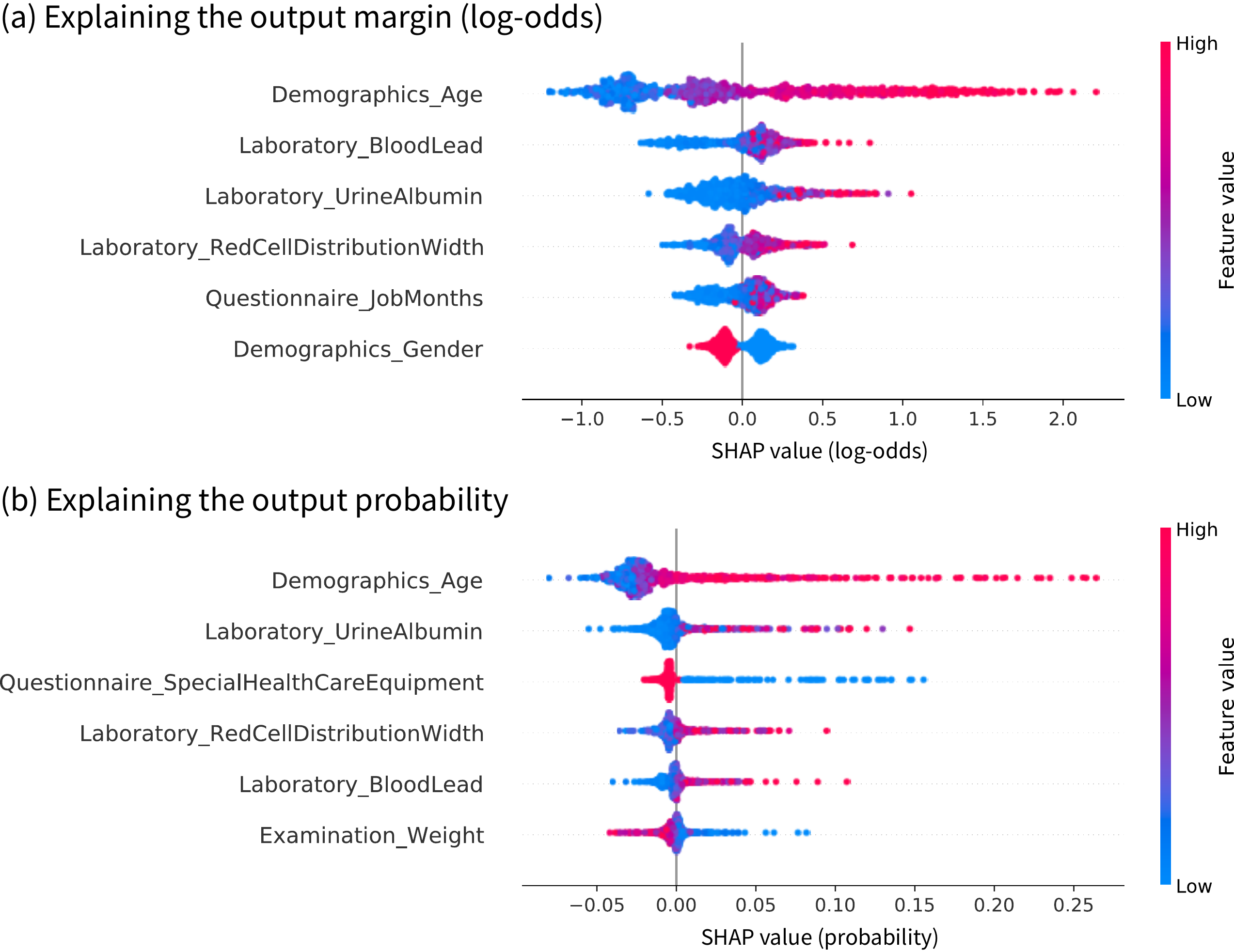}
\centering
\caption{The summary plot for the log-odds model output differs to the summary plot for the probability output in terms of ordering of important features.  This is to be expected because of the non-linear mapping between log-odds and probability.  Often times, it can be useful to communicate scientific findings in terms of the probability output of the model, although the log-odds output is also natural as it is the output margin.}
\label{fig:supp:summary_prob_lodds}
\end{figure}

\subsubsection{Additional gene sets}
\label{sec:supp:additional_gene_sets}

We present attributions aggregated by the Reactome canonical pathway gene set and the Biological Process gene ontology gene set in Figure \ref{fig:supp:additional_pathways}.

\begin{figure}[!ht]
\includegraphics[width=\textwidth]{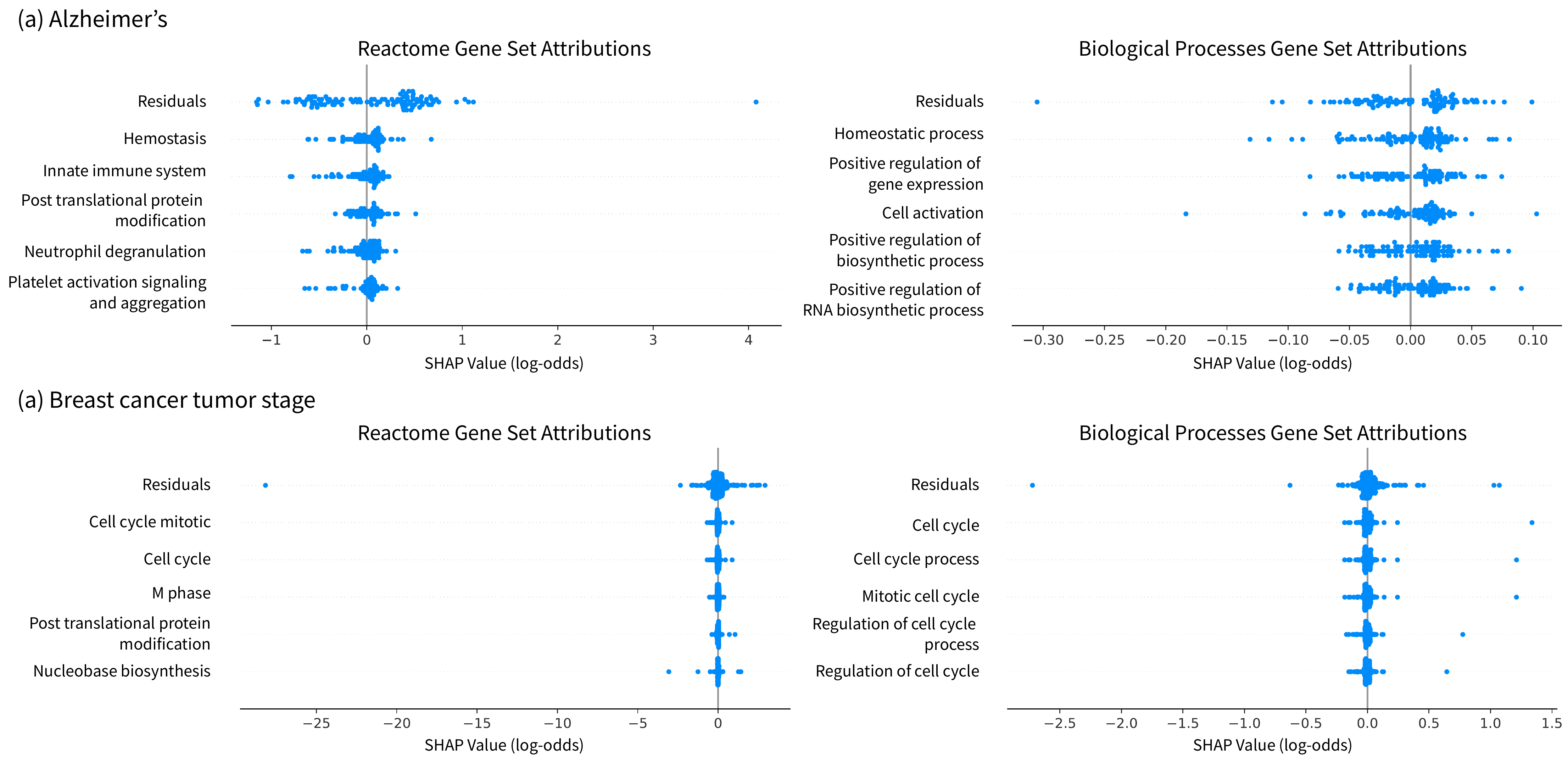}
\centering
\caption{Additional gene set attributions for the Reactome canonical pathway gene set and the Biological Process gene ontology gene set. Analogous to the attributions in Figure \ref{fig:pathway_attributions}}
\label{fig:supp:additional_pathways}
\end{figure}

% \subsubsection{Bias by size of gene set}
% \label{sec:supp:gene_set_size_bias}

% \hugh{Todo}

\subsubsection{Improved predictive performance of feature extraction}
\label{sec:supp:feature_extraction_performance}

In Figure \ref{fig:supp:compare_xgb_to_cnn} we demonstrate the efficacy of deep feature extraction fed into a tree model for MNIST.

\begin{figure}[!ht]
\includegraphics[width=.6\textwidth]{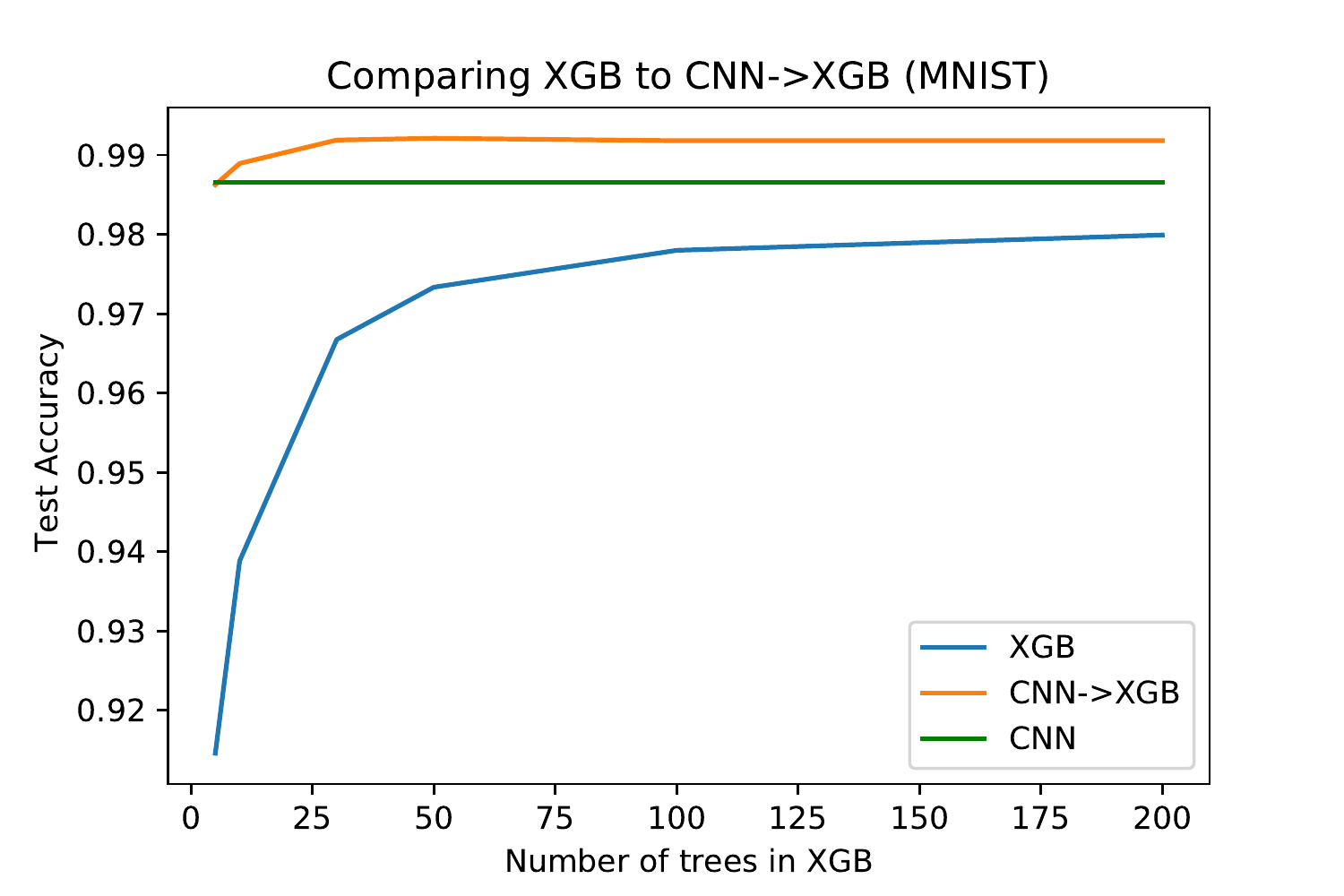}
\centering
\caption{Investigating the utility of CNN feature extraction in MNIST.  We compare a CNN on the raw digits to an GBT model trained on the raw digits to an GBT model trained to classify digits on the basis of the features extracted by the trained CNN.  We vary the number of estimators in GBToost to investigate how well underparamterized trees classify digits with different features.  Overall, using GBT models with the features extracted from the CNN yield much higher accuracy for trees with the same number of estimators.}
\label{fig:supp:compare_xgb_to_cnn}
\end{figure}

\subsubsection{Stacked generalization}
\label{sec:supp:stacked_gen}

\begin{figure}[!ht]
\includegraphics[width=\textwidth]{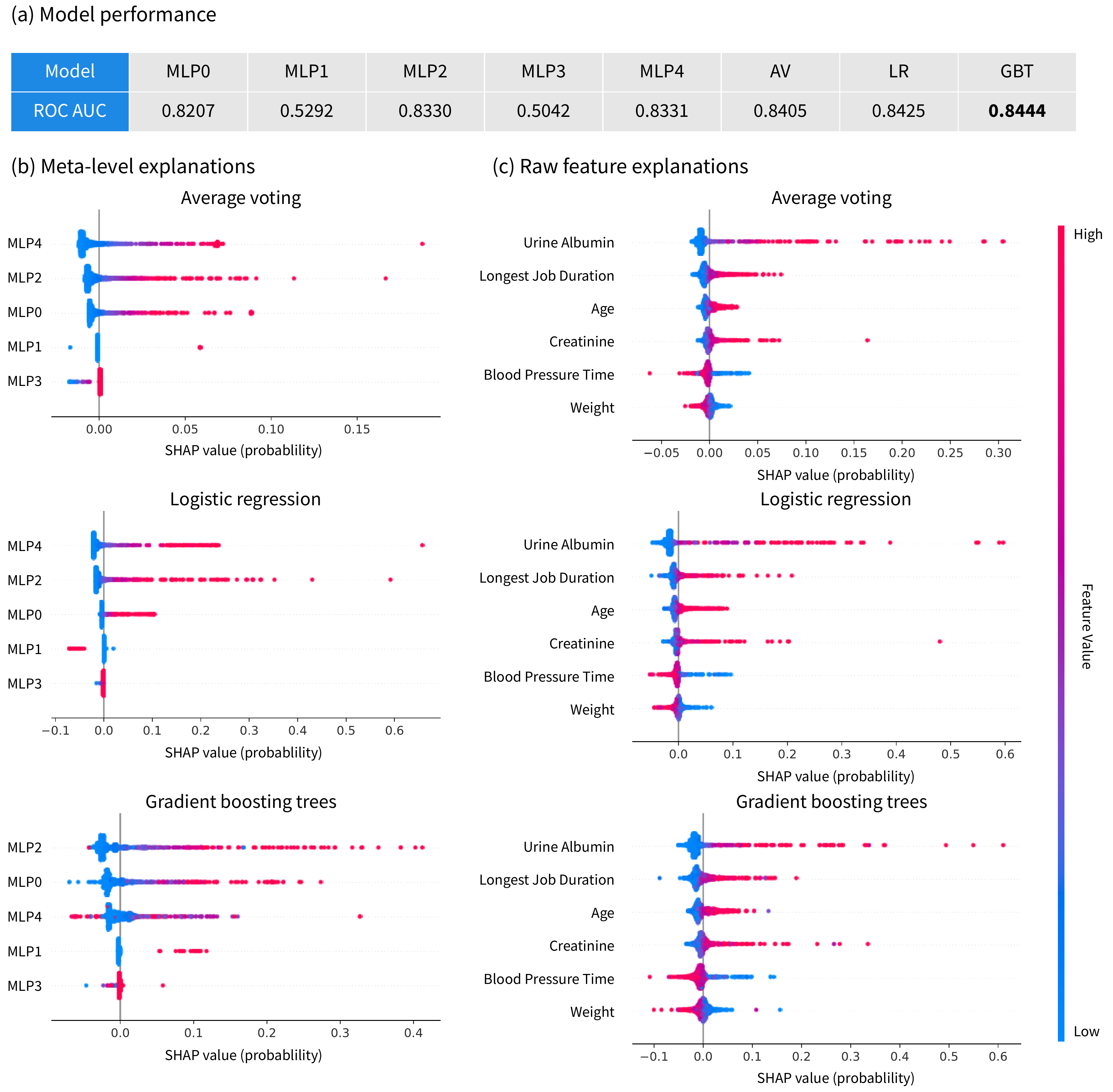}
\centering
\caption{\textbf{Explaining stacked generalization by looking at meta-level and raw feature explanations.} (a) The test set performance of the five MLP models and three meta-models that use make predictions based on the MLP models' predictions.  (b) Intermediary explanations for the meta-models that assign credit based on which MLP was important to the meta-model's prediction.  (c) Raw feature explanations obtained by propagating the credit for each meta-model in (b) to the original feature space.  (b) and (c) show summary plots (Appendix Section \ref{sec:supp:summary_plot}).}
\label{fig:supp:stacked_gen}
\end{figure}

We compare five bagged MLP base-learners (feature attributions in Figures \ref{fig:baselearner0}-\ref{fig:baselearner4}) and three meta-learners (average voting, logistic regression, and gradient boosting trees) that use the base-learners' predictions as features for NHANES (1999-2014) mortality prediction with performance in Figure \ref{fig:supp:stacked_gen}a.  We see that average voting outperforms any individual MLP and is improved upon by a non-uniform weighting scheme (logistic regression).  Finally, stacked generalization with a gradient boosted tree meta-model outperforms both linear approaches.  

Since our framework enables attributions that satisfy efficiency at each layer, we obtain the importance each meta model assigns to each base-learner (Figure \ref{fig:supp:stacked_gen}b), which is much harder to do for model-agnostic methods because it will require separately estimating the importance for each layer.  Although the average voting scheme assigns equal importance to each base model, each MLP's predictions are different, leading to the different shapes in the summary plots.  In contrast, the logistic regression model downweights MLP0 and MLP3 and primarily relies on MLP2 and MLP4 which achieved the highest performance.  The gradient boosting tree model uses the base-learners in a non-linear fashion.  For MLP4, high predictions actually decreases the overall prediction of the meta-learner.  These meta-level explanations reveal novel insights that explanations in the original feature space would not.  Finally, we can also propagate the meta-level explanations back to the original input space and verify that most models give similarly reasonable feature attributions in Figure \ref{fig:model_stack}c.  

\begin{figure}[!ht]
\includegraphics[width=.6\textwidth]{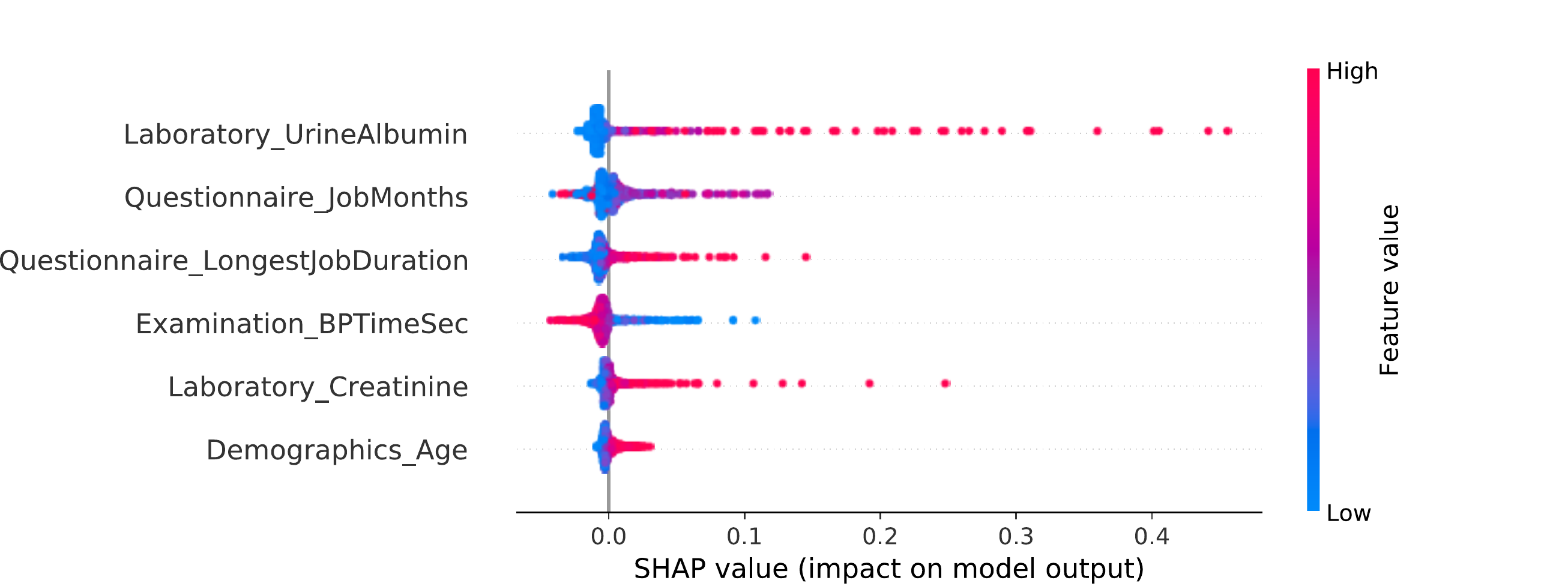}
\centering
\caption{Feature attributions for base learner MLP0.}
\label{fig:baselearner0}
\end{figure}

\begin{figure}[!ht]
\includegraphics[width=.6\textwidth]{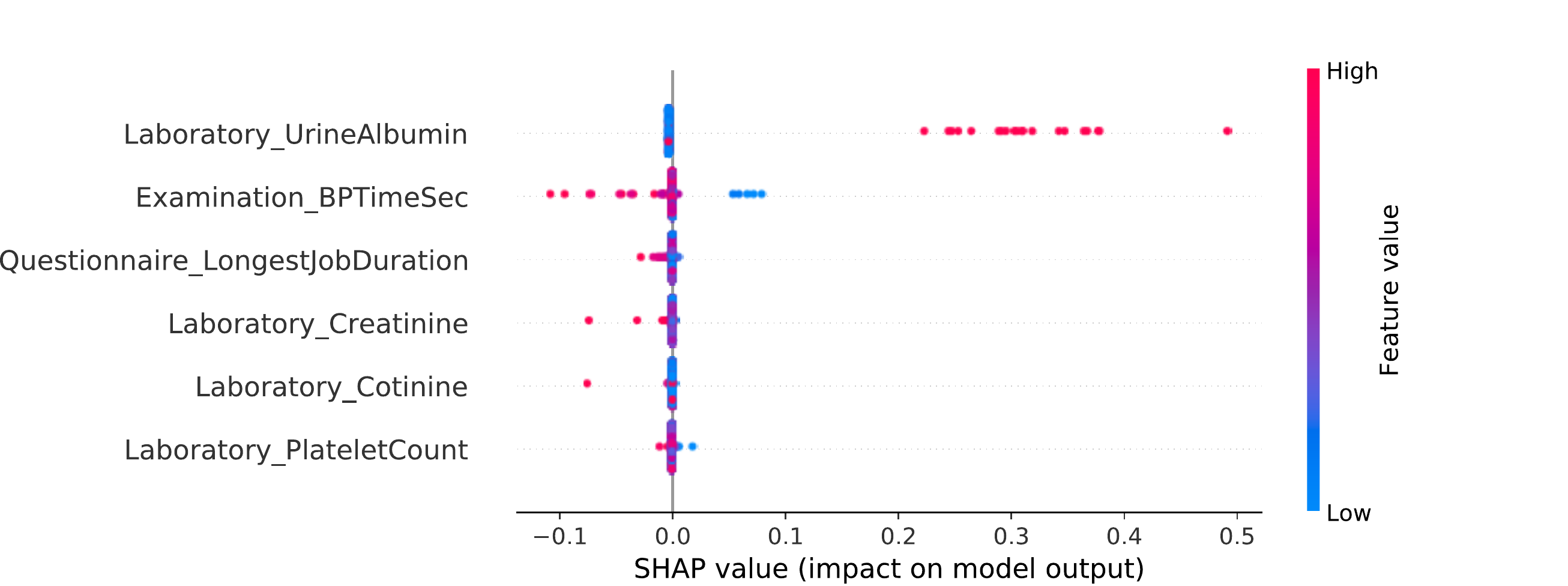}
\centering
\caption{Feature attributions for base learner MLP1.}
\end{figure}

\begin{figure}[!ht]
\includegraphics[width=.6\textwidth]{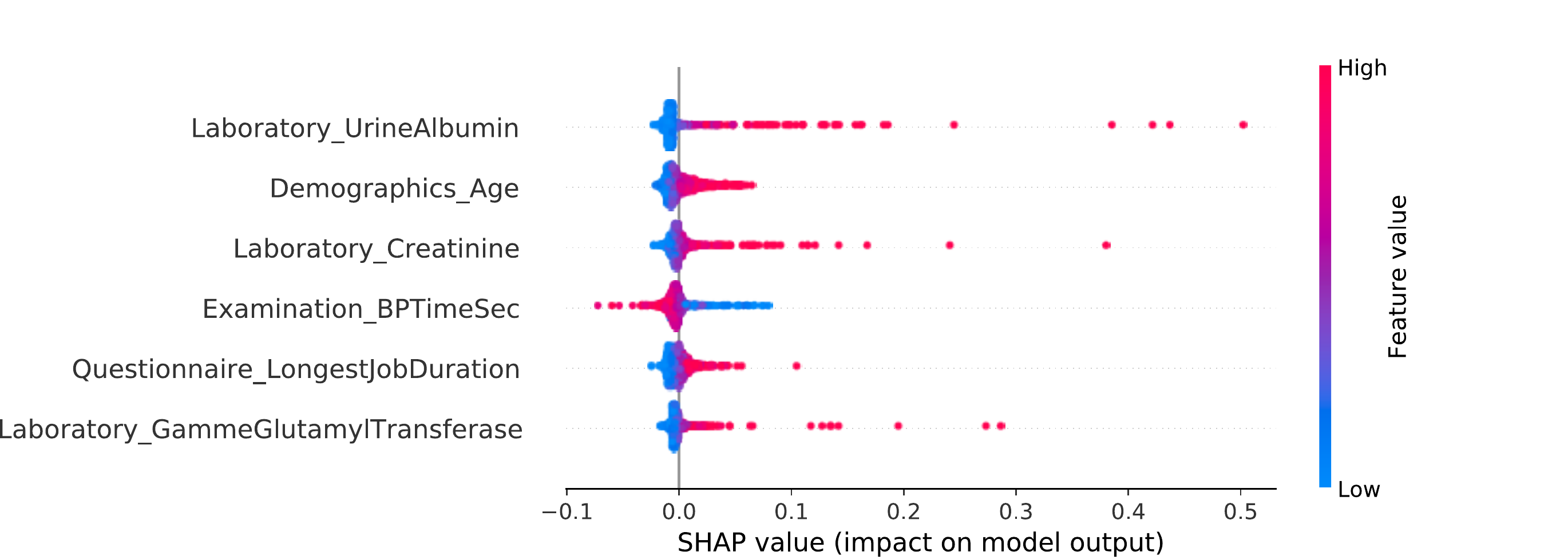}
\centering
\caption{Feature attributions for base learner MLP2.}
\end{figure}

\begin{figure}[!ht]
\includegraphics[width=.6\textwidth]{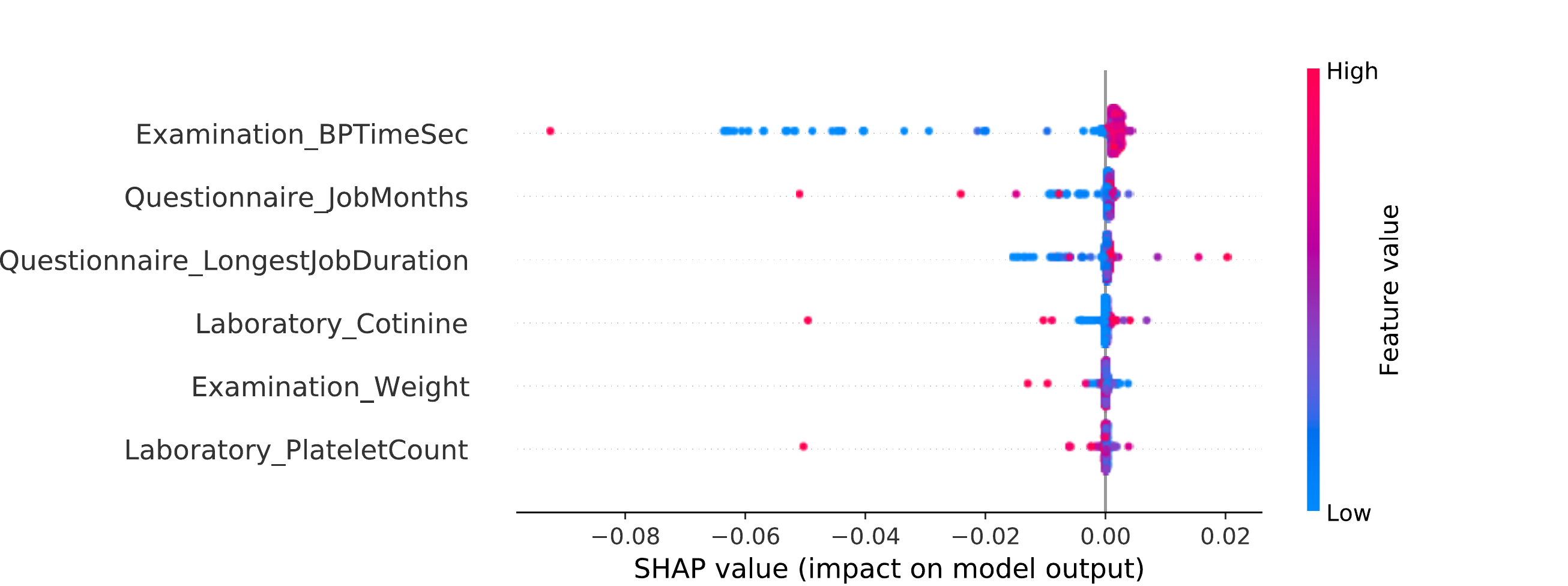}
\centering
\caption{Feature attributions for base learner MLP3.}
\end{figure}

\begin{figure}[!ht]
\includegraphics[width=.6\textwidth]{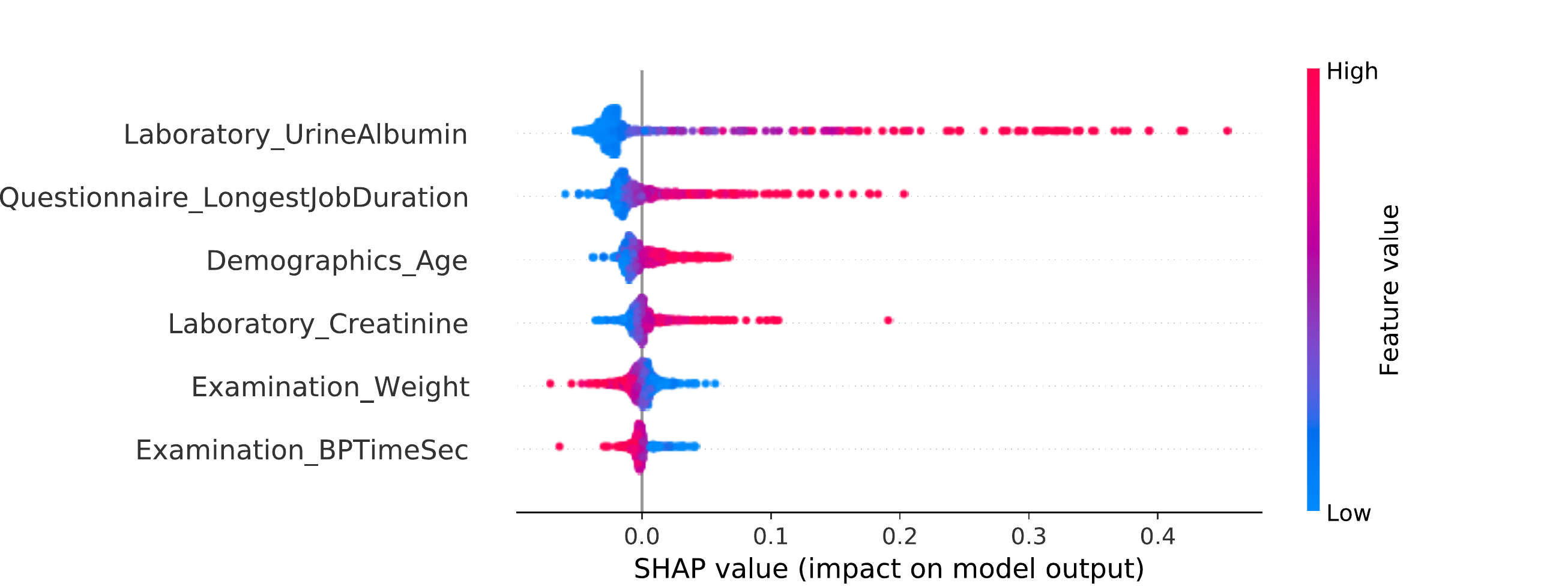}
\centering
\caption{Feature attributions for base learner MLP4.}
\label{fig:baselearner4}
\end{figure}

\subsection{Methods}

\subsubsection{Shapley value axioms}
\label{sec:supp:shapley_axioms}

The Shapley values satisfy a number of desirable properties in terms of the set function $v$.  It is uniquely defined by three axioms:
\begin{itemize}
    \item \textbf{Efficiency:} The sum of the Shapley values for each player equals the value of the game with the set of all players (the grand coalition):
    \begin{equation}
        \sum_{i=1}^m \phi_i(v)=v(M) - v(\emptyset)
    \end{equation}
    \item \textbf{Monotonicity:} If a player $i$ always increases game $v_1$’s value more than they would company $v_2$ for all possible remaining sets of players, then $i$’s attribution for $v_1$ should be greater than or equal to their attribution in $v_2$:
    \begin{equation}
        v_1(S\cup {i}) -v_1(S) \geq v_2(S\cup {i}) - v_2(S) \forall S \subseteq N\setminus {i} \implies \phi_i(v_1)\geq \phi_2(v_2)
    \end{equation}
    \item \textbf{Missingness:} Employees $i$ that don’t help or hurt the company’s profit must have no attribution:
    \begin{equation}
        v(S\cup {i})=v(S)\forall S \subseteq N\setminus {i} \implies \phi_i(v)=0
    \end{equation}
\end{itemize}

While the above three axioms determine the Shapley values as a unique solution concept for credit allocation, the Shapley values have a number of additional desirable properties:
\begin{itemize}
    \item \textbf{Symmetry:} If two players have the same marginal impact for all subsets, then they should have the same Shapley value:
    \begin{equation}
        v(S\cup {i})=v(S\cup {j}) \forall S \subseteq N\setminus {i,j} \implies \phi_i(v) = \phi_j(v)
    \end{equation}
    \item \textbf{Linearity:} The Shapley values for a linear combination of games is equal to the linear combination of Shapley values for each game:
    \begin{equation}
        \phi_i(v_1+v_2) = \phi_i(v_1)+\phi_i(v_2)
    \end{equation}
    and
    \begin{equation}
        \phi_i(av) = a\phi_i(v)
    \end{equation}
\end{itemize}

\subsubsection{Baseline distribution proof for interventional Shapley values}
\label{sec:supp:int_baseline_dist_proof}

\begin{proof}
Define $D$ to be the data distribution, $N$ to be the set of all features, and $f$ to be the model being explained.  Additionally, define  $\mathcal{X}(x,x',S)$ to return a sample where the features in $S$ are taken from $x$ and the remaining features from $x'$.  Define $C$ to be all combinations of the set $N \setminus \{i\}$ and $P$ to be all permutations of $N \setminus \{i\}$.  Starting with the definition of SHAP values for a single feature: $\phi_i(x)$
\begin{align*}
&= \sum_{S\in C } W(|S|,|N|)(\mathbb{E}_{D}[f(X)|x_{S\cup \{i\}}] {-} \mathbb{E}_{D}[f(X)|x_{S}])\\
&=\frac{1}{|P|}\sum_{S\subseteq P} \mathbb{E}_\mathcal{D}[f(x)|\text{do}(x_{S \cup \{i\}})] {-} \mathbb{E}_\mathcal{D}[\text{do}(f(x)|x_{S})]\\
&= \frac{1}{|P|}\sum_{S\subseteq P}\frac{1}{|D|}\sum_{x'\in D} f(\mathcal{X}(x,x',S\cup \{i\})) {-} f(\mathcal{X}(x,x',S))
\\
&= \frac{1}{|D|}\sum_{x'\in D} \underbrace{\frac{1}{|P|}\sum_{S\subseteq P} f(\mathcal{X}(x,x',S\cup \{i\})) {-} f(\mathcal{X}(x,x',S))}_\text{single baseline SHAP value}
\end{align*}
where the second step depends on an interventional conditional expectation \cite{janzing2019feature} which is very close to Random Baseline Shapley in \cite{sundararajan2020themanyshapleyvalues}).
\end{proof}

\subsubsection{Generalized rescale rule is exact for linear models}
\label{sec:supp:linear_series_of_models}

We define a series of models composed of linear functions: $f_k(x)=B^k \cdots B^2 B^1x$ where $B^i\in \mathbb{R}^{o_i\times m_i}$, $m_1=m$, and $o_k=1$.  If we define $\hat{\phi}$ to return Interventional Shapley values for linear models ($\phi(f,x^e,x^b)=\beta (x^e - x^b)$ where $x^e$ and $x^b$ are the inputs to the linear model and $f(x)=\beta x$ \cite{chen2020true}).  Then, the generalized rescale rule gives:
\begin{align}
\psi^k&=B^k (f_{k-1}(x^e)-f_{k-1}(x^b))\\
\psi^i&=B^i (f_{i-1}(x^e)-f_{i-1}(x^b)) \big(\psi^{i+1}\oslash(f_i(x^e)-f_i(x^b))\big),\text{ }i\in 1,\cdots, k-1
\end{align}
Therefore,
\begin{equation}
\phi_i(f_k,x^e,x^b)=B^k \cdots B^2 B^1 (x^e - x^b)
\end{equation}
This coincides with the interventional Shapley values for $f_k(x)$ since the composition of linear models is linear.

\subsubsection{Explaining ensembles of models}
\label{sec:supp:ensembles_of_models}

Explaining ensembles of models is straightforward for Shapley values, because of the linearity property (Appendix Section \ref{sec:supp:shapley_axioms}).  In particular, bagged and boosted ensembles are linear functions of individual models:
\begin{equation}
f(x) = \beta_1 f_1(x) + \cdots \beta_k f_k(x),
\end{equation}
where bagged ensembles have $\beta_i = \frac{1}{k}$ and boosted ensembles have $\beta_i = 1$.  In order to explain these models, it suffices to explain the ensemble model $f$ with:
\begin{equation}
\phi_i(f) = \beta_1 \phi_i(f_1) + \cdots \beta_k \phi_k(f_k)
\end{equation}
Furthermore, explaining linear meta-models for stacked ensembles is also encompassed by the linearity property of Shapley values.  In contrast, in order to explain stacked ensembles with mixed model types as in Section \ref{sec:stacked_generalization}, we employ our generalized rescale rule.

\end{document}